\documentclass[]{scrartcl}

\usepackage{amsfonts,amsmath,amssymb,amsthm}
\usepackage{enumitem}
\usepackage{graphicx}
\usepackage{tabularx}
\usepackage{chngcntr}

\usepackage{natbib}

\theoremstyle{plain}
\newtheorem{cor}{Corollary}
\newtheorem{prop}{Proposition}
\newtheorem{assum}{Assumption Set}
\newtheorem{theorem}{Theorem}
\newtheorem{remark}{Remark}
\newtheorem{lemma}{Lemma}
\newtheorem{example}{Example}

\newcommand{\convdist} {\overset{D}{\rightarrow}}

\newcommand{\gauss}[1]{\mathcal{N}\left(#1\right)} 

\newcommand{\vertiii}[1]{{\left\vert\kern-0.25ex\left\vert\kern-0.25ex\left\vert #1 \right\vert\kern-0.25ex\right\vert\kern-0.25ex\right\vert}}
\newcommand{\set}[1]{ \left\{ #1 \right\}}
\DeclareMathOperator*{\argmin}{argmin\,}
\newcommand{\nablajr}{\nabla}

\newcommand{\prox}{\operatorname{prox}}
\newcommand{\zhat}{{\xi_\epsilon}}

\newcommand{\invInf}{V_{\theta^*}^{-1}}
\newcommand{\Inf}{V_{\theta^*}}

\newcommand{\norm}[1]{\Vert #1 \Vert}
\newcommand{\W}[1]{W_{\theta^*}(#1)}

\newcommand{\loss}{f}
\newcommand{\comment}[1]{}
\DeclareMathOperator{\Tr}{Tr}
\newcommand{\expect}[1]{\mathbb{E}\left[#1\right]}

\newcommand{\errorFirst}[1]{\xi_{-1/2}(#1)}
\newcommand{\errorSecond}[1]{\xi_{-1/2}(#1)  + \xi_{-1}(#1)}
\newcommand{\errorThird}[1]{\xi_{-1/2}(#1)  + \xi_{-1}(#1) + \xi_{-3/2}(#1) }

\newcommand{\biasSecond}[1]{B_2(#1)}
\newcommand{\MSESecond}[1]{M_2(#1)}

\newcommand{\centralized}{\hat{\theta}_N}
\newcommand{\parallelized}{\bar{\theta}}
\newcommand{\machinewise}{\hat{\theta}_n}
\newcommand{\machine}[1]{\machinewise^{(#1)}}

\def\url#1{\expandafter\string\csname #1 \endcsname}
\newcommand{\github}{{https://github.com/johnros/ParalSimulate}}

\newcommand{\error}{\mathcal{E}}

\begin{document}

\title{On the Optimality of Averaging in Distributed Statistical Learning}


\author{
        {
        \sc Jonathan Rosenblatt}$^*$,\\[2pt]
        $^*${\texttt{jonathan.rosenblatt@weizmann.ac.il}}\\[2pt]
        {\sc Boaz Nadler}\\
        {boaz.nadler@weizmann.ac.il}\\[6pt]
        Department of Computer Science and Applied Mathematics, \\ 
        Weizmann Institute of Science, \\
        Rehovot, Israel {\sc}
        }

\maketitle

\begin{abstract}{
A common approach to statistical learning with big-data is to randomly split it among $m$ machines and learn the parameter of interest by averaging the $m$ individual estimates. 
In this paper, focusing on empirical risk minimization, or equivalently M-estimation, we study the statistical error incurred by this strategy.
We consider two large-sample settings: 
First, a classical setting where the number of parameters $p$ is fixed, and the number of samples per machine $n\to\infty$.
Second, a high-dimensional regime where both $p,n\to\infty$ with $p/n \to \kappa \in (0,1)$. 
For both regimes and under suitable assumptions, we present {\em asymptotically} {\em exact} {expressions}  for this estimation error. In the fixed-$p$ setting, under suitable assumptions, we prove that to leading order averaging is \textit{as accurate as} the centralized solution. 
We also derive the second order error terms, and show that  these can be non-negligible, notably for non-linear models. 
The high-dimensional setting, in contrast, exhibits a qualitatively different behavior:
data splitting incurs a first-order accuracy loss, which to leading order increases linearly with the number of machines. The dependence of our error approximations on the number of  machines traces an interesting accuracy-complexity tradeoff, allowing 
the practitioner an informed choice on the number of machines to deploy. 
Finally, we confirm our theoretical analysis with several simulations.  }
\end{abstract}


\section{Introduction}
\label{sec:introduction}

The Big-data era, characterized by huge datasets and an appetite for new scientific and business insights, often involves 
learning statistical models of great complexity. 
Typically, the storage and analysis of such data cannot be performed on a single machine. 
Several platforms such as Map-Reduce~\citep{dean_mapreduce:_2008}, Hadoop~\citep{shvachko_hadoop_2010}, and Spark~\citep{zaharia_spark:_2010} have thus become standards for distributed learning with big-data. 

These platforms allow learning in an ``embarrassingly parallel'' scheme, whereby a large dataset with $N$ observations is split to $m$ machines, each having access to only a subset of $n=N/m$ samples. 
Approaches to ``embarrassingly-parallel'' learning can roughly be categorized along the output of each machine: predictions, parameters or gradients.
In this paper we consider the second, whereby each of the $m$ individual machines fits a model with $p$ parameters and transmits them to a central node for merging. 
This split-and-merge strategy, advocated by \cite{mcdonald_efficient_2009} for striking the best balance between accuracy and communication, is both simple to program and communication efficient: only a single round of communication is performed and only to a central node. 
It is restrictive in that machines do not communicate between themselves, and 
splitting is done only along observations and not along variables. For an overview of more general distributed learning strategies see for example \citet{bekkerman_scaling_2011}.

Our focus is on the \textit{statistical properties} of this split-and-merge approach, under the assumption that the data are split uniformly at random among the \(m\) machines. In particular,  we study the simplest merging strategy, of averaging the $m$ individual machine estimates, denoted as the Mixture Weight Method in \cite{mcdonald_efficient_2009}. In this context we ask the following questions: (i) what is the estimation error of simple averaging as compared to a centralized solution?
(ii) what is its distribution? 
(iii) under which criteria, if any, is averaging optimal? and (iv) how many machines to deploy?

Mcdonald et al. were among the first to study some of these issues for multinomial regression (a.k.a. Conditional Maximum Entropy), deriving finite sample bounds on the expected error of the averaged estimator \cite[Theorem 3]{mcdonald_efficient_2009}.
In a follow-up work, \citet{zinkevich_parallelized_2010} compared the statistical properties of the averaged estimator to the centralized one for more general learning tasks, assuming each machine estimates the model parameters by stochastic gradient descent. 
More recently,  under appropriate conditions and for a large class of loss functions, \citet[Theorem 1]{zhang_communication-efficient_2013} derived bounds for the leading order term in the mean squared error (MSE) of the averaged estimator and provided the rates of higher order terms.
They further proposed several improvements to the simple averaging strategy that reduce the second order term in the MSE, and reduce the machine-wise run time via modified optimization algorithms.  

In this paper we extend and generalize these previous works in several aspects. 
First, in Section~\ref{sec:fixed_p} we study the statistical properties of the averaged estimator, when the number of parameters $p$ is fixed, under conditions similar to those  of \cite{zhang_communication-efficient_2013}. 
Using the classical statistical theory of M-estimators \citep{vaart_asymptotic_1998,rieder_robust_2012}, we provide not only asymptotic bounds on the MSE, but rather an asymptotic expansion of the error itself.
This allows us to derive the exact constants in the MSE expansion, and prove that  as $n \to \infty$, the MSE of the averaging strategy in fact equals that of the centralized solution.
Put differently, for various learning tasks, when the number of machines $m$ and their available memory are such that in each machine there are many observations per parameter ($n\gg p$), then averaging machine-wise estimates is as accurate as the centralized solution. Furthermore, if the centralized solution enjoys first-order statistical properties such as efficiency or robustness, then so will the parallelized solution.
We remark that for maximum likelihood problems, independently of our work, the asymptotic agreement between centralized and averaged estimators was also noted by \citet{liu_distributed_2014_2}.
The asymptotic representation of the averaged estimator also readily yields its limiting distribution. 
This allows to construct  confidence intervals,  perform hypothesis tests on the unknown parameters and feature selection without the need for computationally intensive procedures such as Bootstrapping. 

The first-order equivalence between the averaged and centralized estimators may seem as a free lunch: run-time speedups with no accuracy loss. Distributed estimation via split-and-average, however, does incur an accuracy loss captured in the higher order error terms.  
The classical theory of M-estimators permits the derivation of these terms, in principle up to an arbitrary order. We do so explicitly up to second order, revealing the accuracy loss of split-and-average schemes. 

In Section~\ref{sec:high_dim} we consider the statistical effects of data-splitting in a high-dimensional regime, where the model dimension $p$, grows with the number of observations $n$: $p,n\to\infty$ with $p/n \to \kappa \in (0,1)$.
Our motivation comes from modern day data analysis practices, where increasingly complicated models are considered as more data is made available. This is a challenging regime in that typically machine-wise estimates are not only inconsistent, but in fact do not even converge to deterministic quantities. 
Here, in the absence of a general theory of M-estimators, we restrict our analysis to generative linear models.
In contrast to the fixed-$p$ setting, in this high-dimensional regime there is a first order accuracy loss due to the split data, which increases (approximately) linearly with the number of machines. 
Luckily, in several practical situations, this accuracy loss is moderate. 
Our analysis builds upon the recent results of \cite{el_karoui_robust_2013} and \cite{donoho_high_2013}, and to the best of our knowledge, is the first to study the error loss of parallelization in this high-dimensional regime. 

In Section \ref{sec:simulations} we present several simulations both in the fixed-$p$ and in the high-dimensional regime that illustrate the utility but also the limitations of our results. These confirm that when learning linear models with abundant data, random splitting and averaging is attractive both computationally and statistically. In contrast, for non-linear models, the accuracy loss due to data splitting can be considerable. 

In Section \ref{sec:how_many_machines} we attend to practical considerations, such as parallelization vs. sub-sampling and the choice of number of machines, $m$. 
For the latter, we distinguish between parallelization due to memory constraints, and that motivated by run-time speedups. 
For these two scenarios we formulate the choice of $m$ as optimization problems constrained on the desired error level. 
Interestingly, when motivated by run-time speedups, using our approximations for the estimation error, and varying $m$ traces the accuracy-complexity tradeoff facing the practitioner. 

We conclude with a discussion and several further insights in Section \ref{sec:discussion}. All proofs appear in the 
appendices.

\section{Problem Setup}
\label{sec:setup}

We consider the following general statistical learning setup: 
Let $Z$ be a random variable defined on an instance space $\mathcal{Z}$ and having an unknown density $p_Z$. 
Also, let the parameter space $\Theta \subset \mathbb{R}^p$ be an open convex subset of Euclidean space, and let $f:\mathcal Z\times \Theta\to\mathbb{R}^+$ denote a loss function. Our interest is to estimate the $p$-dimensional parameter $\theta^*\in\Theta$ that minimizes the population risk
\begin{equation}
                \label{eq:risk}               
        R(\theta)=\mathbb{E}_Z[f(Z,\theta)] = 
        \int f(z,\theta) \: p_Z(z)\: \mathrm{d}z.
\end{equation} 
In the following, we assume that $\theta^*$ exists in $\Theta$ and is unique.
Given $N$ i.i.d. samples $\{z_i\}_{i=1}^N$ of the r.v. $Z$, a standard approach, known as M-estimation or empirical risk minimization (ERM), is to calculate the estimator $\hat\theta_N \in \Theta$ that minimizes the empirical risk
\begin{equation}
\label{eq:R_N}
        \hat R_N(\theta) = \frac{1}{N} \sum_{i=1}^N f(z_i,\theta).
\end{equation} 
This  framework covers many common unsupervised and supervised learning tasks.
In the latter,  $Z=(X,Y)$ consists of both features $X$ and labels $Y$.
There is by now an established theory providing conditions for $\hat\theta_N$ to be a consistent estimator of $\theta^*$, and non asymptotic bounds on its finite sample deviation from $\theta^*$ (see ~\cite{devroye_probabilistic_1997,shalev-shwartz_understanding_2014} and references therein). 

In this paper we consider a big-data setting, whereby the number of samples $N$ is so large that instead of minimizing Eq.(\ref{eq:R_N}) on  a single machine, the data is randomly allocated among $m$ machines, each having access to only a subset of size $n:=N/m$. In line with the Map-Reduce workflow, a typical approach in this distributed scenario is that each machine computes its own M-estimator and transmits it to a central node for further processing.
In this work we focus on the most common aggregation procedure, namely simple averaging
\begin{align}
        \label{eq:bar_theta}
    \parallelized &:= \frac{1}{m}\sum_{j=1}^m \machine{j}
\end{align}
where $\machine{j}$ denotes the $j$-th machine minimizer of Eq.~(\ref{eq:R_N}) over its own observed data.

Our questions of interest are:
(i) what is the accuracy of $\parallelized$ vs. that of $\centralized$ ?
(ii) what are the statistical properties of $\parallelized$? (iii) under which criteria, if any, is $\parallelized$ optimal? and (iv) how many machines to deploy?


\section{Fixed-$p$ Setting}\label{sec:fixed_p}

First, we consider the error of the split-and-average estimator $\parallelized$ of Eq.(\ref{eq:bar_theta}), when data is abundant and the model dimension $p$  and number of machines $m$ are both fixed. 
In this setting, bounds on the $MSE[\parallelized,\theta^*] := \mathbb{E}[\Vert \parallelized-\theta^* \Vert^2]$ were derived by both \citet{zhang_communication-efficient_2013} and \cite{mcdonald_efficient_2009}.
For the particular case of  maximum likelihood estimation, \citet[Theorem 4.6]{liu_distributed_2014_2} derived the exact asymptotic expression of the first two leading error terms in the MSE, as $n\to\infty$. 
We take a similar approach but for the more general M-estimators. Instead of focusing on the MSE, we derive an exact asymptotic representation of the first two terms in the error $\parallelized-\theta^*$ itself.

\subsection{First Order Statistical Properties of Averaging}
\label{sec:first_order_properties}

We start by analyzing the exact asymptotic expression for the dominant error term. We make the following standard assumptions \cite[Theorem  5.23]{vaart_asymptotic_1998}, similar to those made in \cite{zhang_communication-efficient_2013}:
\begin{assum}
\label{asum:first_order}
\leavevmode

\begin{enumerate}[label=A\arabic{*}, ref=(A\arabic{*}),leftmargin=5.0em]
        \item $\hat\theta_n$ is consistent: $\machinewise=\theta^* + o_P(1)$.\label{as:consistency}
        \item $R(\theta)$ admits a second order Taylor expansion at $\theta^*$ with non singular Hessian $V_{\theta^*}$.\label{as:taylor}
        \item $f(Z,\theta)$ is differentiable at $\theta^*$ almost surely (a.s.) or in probability. \label{as:as_differentiability}
        \item $f(Z,\theta)$ is Lipschitz near $\theta^*$: 
$| f(Z,\theta_1) - f(Z,\theta_2)|\leq M(Z) \Vert \theta_1-\theta_2 \Vert$ with Lipschitz coefficient $M(Z)$ bounded in squared expectation, $\mathbb{E}[ M(Z)^2] <\infty$.

\end{enumerate}
\end{assum}
Our first result, formally stated in the following theorem, is that under Assumption Set \ref{asum:first_order} 
averaging machine-wise estimates enjoys the \emph{same first-order statistical properties} as the centralized solution.

%

\begin{theorem}
\label{thm:fixed_p_loss}
Under Assumption Set \ref{asum:first_order},  as $n\to\infty$ with $p$ fixed, and any norm
\begin{align}
\label{eq:unimprovable}
        \frac{\Vert \parallelized - \theta^* \Vert}
                {\Vert \centralized - \theta^* \Vert} = 
        1 + o_P(1).
\end{align}
\end{theorem}


We say that two estimators are {\em first-order  equivalent} if their leading error terms converge to the same limit at the same rate, with the same limiting distribution. Assumption Set \ref{asum:first_order} 
implies that  $\machinewise$ converges to $\theta^*$ at  rate $O(n^{-1/2})$ \citep[Corollary 5.53]{vaart_asymptotic_1998}.
Theorem~\ref{thm:fixed_p_loss} thus directly implies the following: 
\begin{cor}\label{cor:first_order_equivalent}
        The averaged estimator $\parallelized$ is first-order equivalent to the centralized solution $\centralized$.
\end{cor}

\begin{remark}
In practice, Eq.(\ref{eq:R_N}) is minimized only approximately, typically by some iterative scheme such as gradient descent (GD), stochastic gradient descent (SGD), etc.
An important point is that Theorem~\ref{thm:fixed_p_loss} holds not only for the \emph{exact} empirical minimizer $\machinewise$ of Eq.(\ref{eq:R_N}), but also for any \emph{approximate} minimizer $\tilde{\theta}_n$ as long as it satisfies
$\hat{R}_n(\tilde{\theta}_n) \leq \hat{R}_n(\machinewise)+o_p(n^{-1})$ \citep[Theorem 5.23]{vaart_asymptotic_1998}. 
In other words, for Corollary \ref{cor:first_order_equivalent} to hold, it suffices to minimize the empirical risk up to $o_{p}(n^{-1})$ precision.
\end{remark}

Theorem \ref{thm:fixed_p_loss} has important implications on the statistical properties of $\parallelized$, its optimality and robustness
. We discuss these in detail below, but before, let us describe the scope which this theorem covers.

\paragraph{Scope}
As detailed further in Appendix~\ref{apx:proof_scope}, the learning tasks covered by Theorem~\ref{thm:fixed_p_loss} are quite broad, and include:
linear or non-linear regression with $l_2$, Huber, or log likelihood loss;
linear or non-linear quantile regression with continuous predictors;
binary regression where $P(Y=1|X)=\Psi_\theta(X)$ for any smooth $\Psi_\theta$ and $l_2$, log likelihood or Huberized hinge loss\footnote{ A smooth version of the Huber loss \citep{rosset_piecewise_2007}.};
binary hinge loss regression (i.e. SVM regression) with continuous predictors;
unsupervised learning of location and scale.
Furthermore, Theorem~\ref{thm:fixed_p_loss} also covers regularized risk minimization with a {\em fixed} regularization term $J(\theta)$, of the form $\theta^*:=\argmin_\theta \set{R(\theta)+J(\theta)}$, provided that the modified loss function $\tilde f(Z,\theta)=f(Z,\theta)+J(\theta)$ satisfies the required assumptions.

Some learning problems, however, are not covered by Theorem~\ref{thm:fixed_p_loss}. 
Examples include:
non-uniform allocation of samples to machines; 
non-convex parameter spaces; 
a data driven regularization term; 
non differentiable loss with discrete predictors. 
Also not covered is the $n<p$ regime, in which  \citet{shamir_communication_2013} showed that averaging (denoted there as One Shot Averaging) can, in general, be unboundedly worse than the centralized solution.

\paragraph{On the optimality of averaging.}
Recall that common notions of asymptotic optimality, such as Best Regular and Local Minimax depend only on the leading order error term \citep[Chapter 8]{vaart_asymptotic_1998}. Hence, if the centralized estimator $\centralized$ is optimal w.r.t. any of these criteria, Eq.(\ref{eq:unimprovable}) readily implies that so is the averaged estimate $\parallelized$.   
A notable example, discussed in  \cite[Corollary~3]{zhang_communication-efficient_2013} and in \cite{liu_distributed_2014_2}, is when the loss function is the negative log likelihood of the generative model. The centralized solution, being the maximum-likelihood estimate of $\theta^*$, is optimal in several distinct senses. 
Theorem~\ref{thm:fixed_p_loss} thus implies that $\parallelized$ is optimal as well and the factor $1$ in Eq.(\ref{eq:unimprovable}) cannot be improved.

\paragraph{Robustness.}
An important question in distributed learning is how to handle potential outliers: should these be dealt with at the machine-level, the aggregation level, or both? Recall that the robustness literature mostly considered the construction of estimators having minimal asymptotic variance, under the constraint of bounded influence of individual observations. For estimating the mean of a Gaussian distribution under possible contamination, Huber derived his famous loss function, and proved it to be optimal. As the Huber-loss yields an M-estimator that satisfies the assumptions of Theorem~\ref{thm:fixed_p_loss}, it thus follows that averaging machine-wise robust estimators is optimal in the same sense. 

Hence, if the probability of a high proportion of outliers in any machine is negligible, and machine-failure is not a concern, it suffices to deal with outliers at the machine level alone. In other cases robust aggregation functions should be considered \citep{hsu_loss_2013, feng_distributed_2014}.

\paragraph{Asymptotic Linearity.}
The proof of Theorem~\ref{thm:fixed_p_loss} relies on the asymptotic linearity of the estimator in some non-linear transformation of the samples. 
This is known as the \textit{asymptotic linearity property} and the corresponding transformation is the \textit{Influence Function}. 
Asymptotic linearity holds for several other estimators, including L, R and Minimum Distance. Hence, first-order equivalence of averaging to the centralized solution is rather general. 
It  typically holds for asymptotically Gaussian estimators \citep[Chapter 1,6]{rieder_robust_2012} and has also been observed in other contexts, such as that of particle filters \citep{achutegui_simple_2014}.

\paragraph{Limiting Distribution}
The asymptotic linearity of $\parallelized$ in the influence function immediately offers the following limiting Gaussian distribution: 
 
\begin{cor}[Asymptotic Normality]
\label{cor:limit_dist_fix_p}
Under the assumptions of Theorem~\ref{thm:fixed_p_loss}, when $n \to \infty$ with $p$ fixed, then
$\sqrt{N}(\parallelized-\theta^*)$ converges in distribution to 
$$
        \mathcal{N}\left(0, 
        \invInf \mathbb{E}\left[ \nabla f(\theta^*)\, \nabla f(\theta^*)'\right] \invInf
        \right).
$$ 
\end{cor}

Corollary~\ref{cor:limit_dist_fix_p} allows to construct confidence intervals and test hypotheses on the unknown $\theta^*$. To this end, the asymptotic covariance matrix also needs to be estimated. Plugging any $O(N^{-1/2})$ consistent estimator for the covariance matrix will conserve the asymptotic normality via Slutsky's Theorem.

\subsection{Second Order Terms}
        \label{sec:second_order_terms}

As we show empirically in Section~\ref{sec:simulations}, relatively little accuracy is lost when parallelizing a linear model but much can be lost when the model is non-linear. 
One reason is that the second order error term may be non-negligible. As discussed in Section \ref{sec:how_many_machines}, this term is also imperative when deciding how many machines to deploy, as the first-order approximation of the error does not depend on $m$ for fixed $N$.

Before studying this second order term, let us provide a high level view. 
Intuitively, the first-order term captures estimation variance, which is reduced by averaging. 
The second order term captures also bias, which is not reduced by averaging. 
We would thus expect some second order suboptimality when parallelizing.
Indeed, Theorem \ref{thm:second_order_bias} below shows that the (second order) bias in a parallelized estimator is $m$ times larger than that of the centralized one. 
The comparison between the second order MSE matrix of the parallelized and centralized estimators is more complicated. 
Theorem \ref{thm:second_order_MSE} provides an explicit expression, whose terms ultimately depend on the curvature of the risk \(R(\theta)\) at $\theta^*$.

\subsubsection{Notation and Assumptions}
\label{sec:notation}

To study the second order error of $\parallelized$, we make suitable assumptions that ensure that the machine-wise M-estimator admits the following higher-order expansion,
\begin{align}
        \label{eq:second_order}
        \machinewise &= \theta^* + \errorThird{\machinewise}+ O_P(n^{-2}),
\end{align}
where $\xi_{-\alpha}(\machinewise)$ denotes the $O_P(n^{-\alpha})$ error term in $\machinewise$ and \(\alpha=\{ 1/2,1,3/2,\dots \}\). 
The following set of assumptions with \(s=4\) is sufficient for Eq.(\ref{eq:second_order}) to hold, see \cite{rilstone_second-order_1996}. 

\begin{assum} There exist a neighborhood of \(\theta^*\) in which all of the following conditions hold:
\label{assum:second_order}
\leavevmode

\begin{enumerate}[label=B\arabic{*}, ref=(B\arabic{*}),leftmargin=5.0em]

                \item Local differentiability: 
                 $\nabla ^s f(\theta,Z)$ up to order \(s\), exist a.s.\ and  
                 $\expect{\norm{\nabla^s f(\theta^*,Z)}} < \infty $.

                \item Bounded empirical Hessian: 
                 $( \nablajr^2 \hat{R}_n(\theta))^{-1} =O_P(1)$.
                
                \item Lipschitz gradients: 
                $\norm{\nabla^s f(\theta,Z) - \nabla^s f(\theta^*,Z)} \leq M \norm{\theta- \theta^*} $, where $\expect{|M|} \leq C < \infty $.

\end{enumerate}
\end{assum}

For future use, and following the notation in \cite{rilstone_second-order_1996}, we  define the following  $p \times 1$ column vector  $\delta$, and  \(p\times p\) matrices  $\gamma_0, \dots, \gamma_4$,
\begin{align}
\label{eq:gammas}
        & \expect{\xi_{-1}(\machinewise)}  = n^{-1} \delta;
        && \expect{\xi_{-1}(\machinewise)} \expect{\xi'_{-1}(\machinewise)} = n^{-2} \gamma_0 = n^{-2} \delta \delta';  \nonumber\\
        & \expect{\xi_{-1/2}(\machinewise) \, \xi'_{-1/2}(\machinewise)} = n^{-1} \gamma_1;  
        && \expect{\xi_{-1}(\machinewise) \, \xi'_{-1/2}(\machinewise)} = n^{-2} \gamma_2; \\
        & \expect{\xi_{-1}(\machinewise) \, \xi'_{-1}(\machinewise)} = n^{-2} \gamma_3+ o(n^{-2}); 
        && \expect{\xi_{-3/2}(\machinewise) \, \xi'_{-1/2}(\machinewise)} = n^{-2} \gamma_4 + o(n^{-2}). \nonumber
\end{align}

\subsubsection{Second Order Bias}

Let \(\biasSecond{\machinewise}\) denote the second order bias of \(\hat\theta_n\) w.r.t. \(\theta^*\):
\begin{align}
        \biasSecond{\machinewise}:= \mathbb{E}[\errorSecond{\machinewise}].
\end{align}
The following theorem, proven in Appendix~\ref{apx:proof_fixed_p_second_order_bias}, shows that under our assumptions averaging over \(m\) machines is (up to second order) \(m\) times more biased than the centralized solution.


\begin{theorem}[Second Order Bias]
        \label{thm:second_order_bias}
Under Assumption Set  \ref{assum:second_order} with $s=3$, 
$
        \biasSecond{\centralized} =
        \delta / N
$   
and 
$
        \biasSecond{\parallelized} =
        \delta / n,
$
so that
\begin{align}
        \label{eq:second_order_loss}
        \biasSecond{\parallelized} = m \, \biasSecond{\centralized}.
\end{align} 
\end{theorem}

\begin{remark}
\label{remark:reduce_bias}

The second order bias $\biasSecond{\parallelized}$ can be reduced at the cost of a larger first-order error, i.e., trading bias for variance. 
In general, this should be done with caution, since in extreme cases debiasing may inflate variance infinitely \citep{doss_price_1989}. Approaches to reduce the second order bias include that of \citet{kim_higher_2006} who modifies the machine-wise loss function, \citet{liu_distributed_2014_2}
who propose a different aggregation of the \(m\) machine-wise estimates, and \citet{zhang_communication-efficient_2013}, whose SAVGM algorithm  estimates the machine-wise bias via bootstrap. 
 A different approach is to trade bias for communication. Recent works that reduce the bias by allowing communication between the \(m\) machines include
the DDPCA algorithm \citep{meng_distributed_2012}, which transfers parts of the inverse Hessian between machines
and DANE \citep{shamir_communication_2013}, which transfers gradients. 
\end{remark}

\subsubsection{Second Order MSE}

Following \cite{rilstone_second-order_1996}, for any estimator \(\tilde\theta_n \) based on $n$ samples, we denote by $\MSESecond{\tilde\theta_n}$ its second order MSE matrix,  
\begin{align}
\label{eq:second_order_MSE_definition}
        \expect{(\tilde\theta_n - \theta^*)(\tilde\theta_n - \theta^*)'}= \MSESecond{\tilde\theta_n} +o(n^{-2}).
\end{align}
It follows from \citet[Proposition 3.4]{rilstone_second-order_1996} that under Assumption Set \ref{assum:second_order} with \(s=4\)
\begin{align}
                \label{eq:M2_gamma}
        \MSESecond{\machinewise}= 
                \frac{1}{n} \gamma_1 + 
                \frac{1}{n^2} \left(\gamma_2+\gamma_2'+\gamma_3+\gamma_4+\gamma_4' \right).
\end{align}
The following theorem compares between $M_2(\centralized)$ and $M_2(\parallelized)$.
\begin{theorem}[Second Order MSE]
        \label{thm:second_order_MSE}
        
Under Assumption Set \ref{assum:second_order} with $s=4$, the matrix \(M_{2}(\parallelized)\)
is given by 
\begin{align}
\label{eq:second_order_MSE}
        \MSESecond{\parallelized}= 
                \frac{m-1}{m} \frac{1}{n^2} \gamma_0 +
                \frac{1}{mn} \gamma_1 + 
                \frac{1}{m n^2} \left(\gamma_2+\gamma_2'+\gamma_3+\gamma_4+\gamma_4' \right).
\end{align}
Furthermore, the excess second order error due to parallelization is given by
\begin{align}
       \label{cor:second_order_inneficiency}
                \MSESecond{\parallelized}-\MSESecond{\centralized}=
                \frac{m-1}{m}  \frac{1}{n^2} \gamma_0 + 
                \frac{m-1}{m^2} \frac{1}{n^2} \left(\gamma_2+\gamma_2'+\gamma_3+\gamma_4+\gamma_4' \right). 
\end{align}
\end{theorem}

In general, the second order MSE matrix $M_2(\machinewise)$ of Eq.(\ref{eq:M2_gamma}) need not be positive definite (PD) ~\citep{rilstone_second-order_1996}. Note that since both matrices \(\gamma_0\) and \(\gamma_1\) are PD by definition, a simple condition to ensure that both $M_2(\machinewise)$ and $\MSESecond{\parallelized}-\MSESecond{\centralized}$ are PD\ is that $\left(\gamma_2+\gamma_2'+\gamma_3+\gamma_4+\gamma_4' \right)$ is PD. If this holds, then parallelization indeed deteriorates accuracy, at least up to second order.

\begin{remark}
\label{rem:ridge_anomaly}
Even if the second order MSE matrix is PD, due to higher order terms, parallelization may actually be \textit{more} accurate than the centralized solution. An example is ridge regression with a fixed penalty and null coefficients (i.e., $\theta^*=0$). The regularization term, being fixed, acts more aggressively with $n$ observations than with $N$. The machine-wise estimates are thus more biased towards $0$ than the centralized one. 
As the bias acts in the correct direction,  $\parallelized$ is more accurate than $\centralized$.
Two remarks are, however, in order:
(a) This phenomenon is restricted to particular parameter values and shrinkage estimators. It does not occur uniformly over the parameter space.
(b) In practice, the precise regularization penalty may be adapted to account for the parallelization, see for example \citet{zhang_divide_2013}.

\end{remark}

\subsection{Examples}
\label{sec:examples}

We now apply our results to two popular learning tasks: ordinary least squares (OLS) and ridge regression, both assuming a generative linear model.
We study these two cases not only due to their popularity, but also as they are analytically tractable. 
As we show below, parallelizing the OLS task incurs no excess bias, but does exhibit excess (second order) MSE.
The ridge problem, in contrast, has both excess bias and excess (second order) MSE.

\subsubsection{OLS}
\label{eg:OLS}
Consider the standard generative linear model 
$Y = X' \theta_0 + \varepsilon,$ where the explanatory variable $X$ satisfies
        $\expect{X} = 0;\; 
        Var[X] = \Sigma$, 
and the noise $\varepsilon$ is independent of $X$ with mean zero and  $Var[\varepsilon]=\sigma^2$. 
The loss is 
        $f(Y,X;\theta) = \frac{1}{2}(Y-X'\theta)^2$,
whose risk minimizer is the generative parameter, $\theta^*=\theta_0$. 
The following proposition, proved in Appendix~\ref{apx:proof_OLS_second_moments}, provides explicit expressions for the second order MSE matrix. 

\begin{prop}[OLS Error Moments]
\label{prop:OLS_second_moments}
For the OLS problem, under the above generative linear model, 
\begin{align*}
        \gamma_0 &= 0, &
        \gamma_1 &= \sigma^2 \Sigma^{-1}, &
        \gamma_2 &= - (1+p) \sigma^2 \Sigma^{-1},  \\
        \gamma_3 &= (1+p) \sigma^2 \Sigma^{-1}, & 
        \gamma_4 &= (1+p) \sigma^2 \Sigma^{-1} . &
\end{align*}
\end{prop}

Inserting these expressions into Theorem~\ref{thm:second_order_bias} yields that the second order bias vanishes both for the individual machine-wise estimators and for their average, i.e., 
$\biasSecond{\hat{\theta}_n}=\biasSecond{\hat{\theta}_N}
=\biasSecond{\bar{\theta}}=0$.
Combining Proposition~\ref{prop:OLS_second_moments} with Theorem~\ref{thm:second_order_MSE} yields the following expressions for the parallelized second order MSE and the excess error, 
\begin{align}
\label{eq:second_order_ols}
        M_2(\bar{\theta}) = 
        \frac{1}{mn} \sigma^2 \Sigma^{-1} + \frac{1}{m n^2} (1+p) \sigma^2 \Sigma^{-1} ; \quad
        M_2(\bar{\theta}) - M_2(\hat{\theta}_N) = 
        \frac{m-1}{m n^2} (1+p) \sigma^2 \Sigma^{-1} .
\end{align} 
In OLS, parallelization thus incurs a second order accuracy loss, since $ M_2(\bar{\theta}) - M_2(\hat{\theta}_N)$ is PD.

\subsubsection{Ridge Regression}
\label{eg:ridge}
Next, we analyze ridge regression under the same generative model  $Y = X '\theta_0 + \varepsilon; \;
        \expect{X} = 0;\; 
        Var[X] = \Sigma,$ but now with the ridge penalty
        $f(Y,X;\theta) = \frac{1}{2}(Y-X'\theta)^2 + \frac{\lambda}{2}\Vert \theta \Vert^2. $
The risk minimizer $\theta^*$ now equals $(\Sigma+\lambda I)^{-1}\Sigma \theta_0$.

Adding the simplifying assumption that $\Sigma=I$, and denoting 
$\lambda_{k,l}:=\frac{\lambda^k}{(1+\lambda)^l}$,
$B:=\theta_0 \theta_0'$, and  
$A:= Tr(B) I= \Vert \theta_0 \Vert^2 I$, we obtain the following result.

\begin{prop}[Ridge Error Moments]
\label{prop:Ridge_second_moments}
For the ridge regression problem, under the above conditions, 
the matrices $\gamma_0,\ldots,\gamma_4$ that control the second order bias and MSE of Eq.(\ref{eq:gammas}) are given by
        \begin{align*}
        &\gamma_0 = \lambda_{2,6} (1+p)^2 B  , 
        &&\gamma_1 = \lambda_{2,4} (B+A) + \lambda_{0,2} \sigma^2 I , \\
        &\gamma_2 = - \lambda_{2,5} ((4+p)B+(2+p)A) - \lambda_{0,3} \sigma^2 (1+p) I  , 
        &&\gamma_3 =  \lambda_{2,6} \left(
                (5+p+p^2)B+ (2+p)A \right) +
                \lambda_{0,4} \sigma^2 (1+p) I  , \\
        &\gamma_4 =  \lambda_{2,5} \left(
                (5+2p)B+ (3+2p)A \right) +
                \lambda_{0,3} \sigma^2 (1 + p) I  .      
        \end{align*}

\end{prop}

\begin{cor}[Ridge Second Order Bias]
Combining Proposition \ref{prop:Ridge_second_moments} with Theorem~\ref{thm:second_order_bias},
 under a linear generative model, the second order bias of the parallelized ridge regression estimate is 
$
        B_2(\bar{\theta}) = - \frac{1}{n} \lambda_{1,3} (1+p) \theta_0.
$
\end{cor}

\begin{cor}[Ridge Second Order MSE]
Combining Proposition \ref{prop:Ridge_second_moments} with Theorem~\ref{thm:second_order_MSE}, under a linear generative model, the parallelized second order MSE matrix and excess MSE are given by
\begin{align}
\begin{split}
\label{eq:second_order_ridge}
        M_2(\parallelized) =& 
                \frac{1}{mn}  \left(
                        \lambda_{2,4} (B+A) + \lambda_{0,2} \sigma^2 I
                \right) + \frac{m-1}{m} \frac{1}{n^2} \lambda_{2,6} (1+p)^2  B + \\
                & \frac{1}{mn^2} \left[
                         \lambda_{2,5} 2(p+1) (B+A) + \lambda_{2,6} ((5+p+p^2)B+(2+p)A) + 
                         \lambda_{0,4}(1+p) \sigma^2 I 
                \right],
\end{split}\end{align} 
and
\begin{align*}
                \MSESecond{\parallelized}-\MSESecond{\centralized} =&
                 \frac{m-1}{m} \frac{1}{n^2} \lambda_{2,6} (1+p)^2  B + \\
                & \frac{m-1}{m^2} \frac{1}{n^2} 
                \left[
                         \lambda_{2,5} 2(p+1) (B+A) + \lambda_{2,6} ((5+p+p^2)B+(2+p)A) + 
                         \lambda_{0,4}(1+p) \sigma^2 I 
                \right].
\end{align*}

\end{cor}

As in the OLS case, since  $B$ is an outer product, and $A$ is a scaled identity matrix,  it follows that both $M_2(\parallelized)$ and $\MSESecond{\parallelized}-\MSESecond{\centralized}$ are PD matrices. Despite this result, as discussed in Remark~\ref{rem:ridge_anomaly}, it is still possible that $\expect{(\parallelized-\theta^*)(\parallelized-\theta^*)'}
-\expect{(\centralized-\theta^*)(\centralized-\theta^*)'}$ is a negative definite matrix due to higher order error terms, implying that parallelized ridge regression can be more exact than the centralized estimator. 
This has been confirmed in simulations (not included).

\section{High-Dimensional Approximation}
\label{sec:high_dim}

As reviewed in Section \ref{sec:introduction}, most existing theory on parallelization assumes a fixed-$p$, independent of $n$. This is implied by the assumption that the empirical risk gradients have uniformly bounded moments, independent of $n$ \citep[e.g.][Assumption 3]{zhang_communication-efficient_2013}. However,
it is common practice to enrich a model as more data is made available, to the extent that  the number of unknown parameters is comparable to the number of samples.
If $p$ is comparable to $n$, and both are large, then the approximations of Section~\ref{sec:fixed_p} may underestimate the parallelization's excess error. To address this setting, we now perform a high-dimensional analysis where $n,p(n) \to \infty$ and $p(n)/n \to \kappa \in (0,1)$.
 
To the best of our knowledge there is no general theory for the behavior of M-estimators is this regime.  
To gain insight into the statistical properties of parallelization in this high-dimensional setting, we restrict our focus to {\em generative linear} models for which the appropriate theory has been developed only recently. 
Building on the works of \cite{donoho_high_2013} and \cite{el_karoui_robust_2013}, we thus consider a random variable $Z=(X,Y)$ consisting of a vector of predictor variables ($X \in \mathbb{R}^p$) and a scalar response variable $(Y\in\mathbb{R})$, which satisfy the following assumptions:
\begin{assum}
\leavevmode
\label{asum:high_dim_assumptions}
\begin{itemize}
\item [C1] The observed data \(x_{i}\) are i.i.d. from the random variable $X\sim\mathcal{N}(0,\Sigma)$, with invertible  $\Sigma$. 
\item [C2] Linear generative model: 
$Y = X' \theta_0 + \epsilon$,  
where 
$\theta_{0} \in \mathbb{R}^p$.
\item [C3] The noise random variable $\epsilon$ has zero mean, finite second moment, and is independent of $X$.
\item [C4] The loss $f(Z,\theta)=f(Y-X'\theta)$ is smooth and strongly convex.
\end{itemize}
\end{assum}

Unlike the fixed-$p$ case, in the high-dimensional regime where \(p,n\to\infty\) together, each machine-wise estimate is inconsistent. As shown by \citet{el_karoui_robust_2013}, and \citet{donoho_high_2013}, when Assumption Set \ref{asum:high_dim_assumptions} holds, then as $n,p(n) \to \infty$ with $p(n)/n \to \kappa \in (0,1)$, 
\begin{align}
\label{eq:asymptotic_representation}
        \machinewise = \theta^* + r(\kappa) \, \Sigma^{-1/2} \, \xi \, (1+o_P(1))
\end{align}
where $\xi\sim \mathcal N(0,1/p \times I_p)$ and  $r(\kappa)$ is a deterministic quantity that depends on $\kappa$, on the loss function $f$ and on the distribution of the noise $\epsilon$.

Using the above result, we now show that in contrast to the fixed-$p$ setting, averaging is not even first-order equivalent to the centralized solution. The following lemma, proven in Appendix \ref{apx:proof_perturbation_result}, quantifies this accuracy loss showing that, typically, it is moderate.

\begin{lemma}
\label{lemma:perturbation_result}
Under Assumption Set \ref{asum:high_dim_assumptions}, as $\kappa \to 0$
\begin{align}
\label{eq:preturbation_result}
     \frac{\expect{\norm{\parallelized-\theta^*}^2}}
     {\expect{\norm{\centralized-\theta^*}^2}} &= 
     1 + \kappa \: \frac{r_2}{r_1} \left(1-\frac{1}{m} \right) + O(\kappa^2),
\end{align}
where  
\begin{align}
\label{eq:r_2} 
r_1 &= \frac{B_1}{A_2^2} \,, & 
r_2 &= \frac{3 B_1 T_1}{A_2^4} - \frac{2 B_1^2 A_4}{A_2^5} + \frac{2 B_2}{A_2^3} \,,
\end{align}
and 
\begin{align}
\label{eq:loss_moments}
        & A_2 = \mathbb{E} [f_{[2]}(\epsilon)] ,   
        && A_4 = \mathbb{E}[1/2 \, f_{[4]}(\epsilon)] , 
        & T_1 = \mathbb{E} [f_{[2]}^2(\epsilon) + f_{[1]}(\epsilon) f_{[3]}(\epsilon)] , \\
        & B_1 = \mathbb{E} [f_{[1]}^2(\epsilon)] , 
        && B_2 = \mathbb{E} [f_{[1]}^2(\epsilon) f_{[2]}(\epsilon)] .\nonumber
\end{align}
where $f_{[i]}:=\frac{\partial^i}{\partial t^i} f(t)$.
\end{lemma}

\begin{remark}
For simplicity of exposition, we followed the assumptions of \citet[Result 1]{bean_optimal_2013}.
However, many of these can be relaxed, as discussed by \citet{el_karoui_asymptotic_2013}.
In particular,  
$X$ need not be Gaussian provided that it is asymptotically orthogonal and exponentially concentrating;
$f$ need not be strongly convex nor infinitely differentiable and an interplay is possible between assumptions on the tail mass of $\epsilon$ and $f_{[1]}$.
Note however, that for our perturbation analysis on the behavior of $r(\kappa)$ as $\kappa \to 0$ to hold, we assume the loss is at-least six times differentiable with bounded sixth derivative. 
\end{remark}

\begin{remark}
There are cases where \(r(\kappa)\) can be evaluated exactly, without recurring to approximations.
One such case is least squares loss with arbitrary noise $\epsilon$, satisfying C3, in which Wishart theory gives $r^2(\kappa) = \frac{\kappa}{1-\kappa}\sigma^2$ \citep{el_karoui_robust_2013}. 
A second order Taylor approximation of this exact result yields $r_1=r_2=\sigma^2$ in accord with our Eqs.(\ref{eq:preturbation_result}) and (\ref{eq:r_2}). 
A second case where an exact formula is available is $l_1$ loss with Gaussian errors:
A second order Taylor expansion of the closed form solution derived in \citep[Page 3]{el_karoui_robust_2013} gives $r_2/r_1=0.904$, again consistent with Eq.(\ref{eq:preturbation_result}).
\end{remark}

\paragraph{Typical Accuracy Losses} 
In classical asymptotics where $\kappa \to 0$, Eq.(\ref{eq:preturbation_result}) is consistent with the results of Section~\ref{sec:fixed_p} in that splitting the data has no (first-order) cost.
In practical high-dimensional scenarios, where the practitioner applies the ``no less than five observations per parameter'' rule of thumb~\citep{huber_robust_1973}, the resulting value of $\kappa$ is at most 0.2. The accuracy loss of splitting the data is thus small provided that the ratio $r_2/r_1$ is small. 
As shown in Table \ref{tab:noise_cost}, for several loss functions with either Gaussian or Laplace errors, this ratio is approximately one.

\begin{table}[h]
\centering
\small
\begin{tabular}{|p{2.5cm}|p{4cm}|p{1.2cm}|p{1.2cm}|}
\hline 
Loss & $f(t)$ & Gaussian & Laplace\\ 
\hline 
\hline
Squared & $ t^2/2 $ & $1$ & $1$\\ 
\hline 
Pseudo Huber & $\delta^2(\sqrt{1+(t/\delta)^2}-1)$ \;; $\delta=3$ & $0.92$ & $1.3$\\ 
\hline 
Absolute Loss
& $|t|$ & $0.9$ & $1.83$ \\ 
\hline 
\end{tabular} 
\caption{The ratio $\frac{r_2}{r_1}$ for different loss functions and noise, $\epsilon$, distributions (Gaussian and Laplace). From Eq.(\ref{eq:preturbation_result}), small values of $\frac{r_2}{r_1}$ imply a small accuracy loss when parallelizing. \label{tab:noise_cost}}
\end{table}

\paragraph{Limiting Distribution} 
Similar to the fixed $p$ regime, using Eq.(\ref{eq:asymptotic_representation}) we can derive the following limiting distribution of $\parallelized$ in the high-dimensional regime which is immediate from the results of \citet[p.1 in SI]{bean_optimal_2013}, and the fact that $\parallelized$ has $m$ times less variance than $\machinewise$.
\begin{cor}[Asymptotic Normality]
\label{cor:limit_dist_high_dim}
Under the assumptions of Lemma~\ref{lemma:perturbation_result}, for a fixed contrast $v$, as $p,n \to \infty$ with $p/n \to \kappa \in (0,1)$, 
then 
$$\frac{v' \parallelized- v' \theta^*}{r(\kappa) \sqrt{\frac{v' \Sigma^{-1} v}{pm}}} \convdist \gauss{0,1} .$$ 
\end{cor}

\section{Simulations}
\label{sec:simulations}

We perform several simulations to validate our results and assess their stability in finite samples. 
For reproducibility, the R simulation code is available at \github.
\begin{figure}[t]
        \centering
        \includegraphics[width=0.24\textwidth]{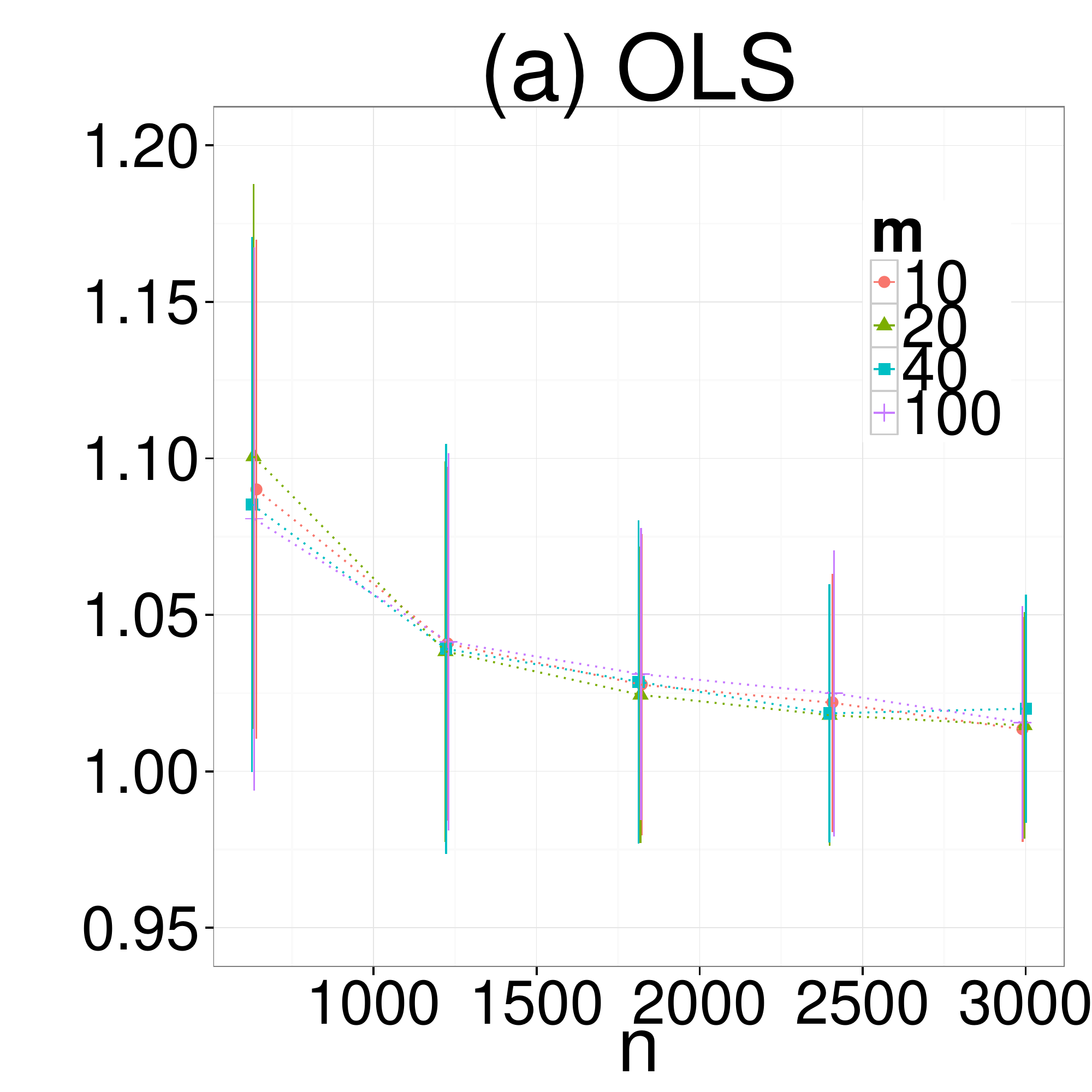}
        \includegraphics[width=0.24\textwidth]{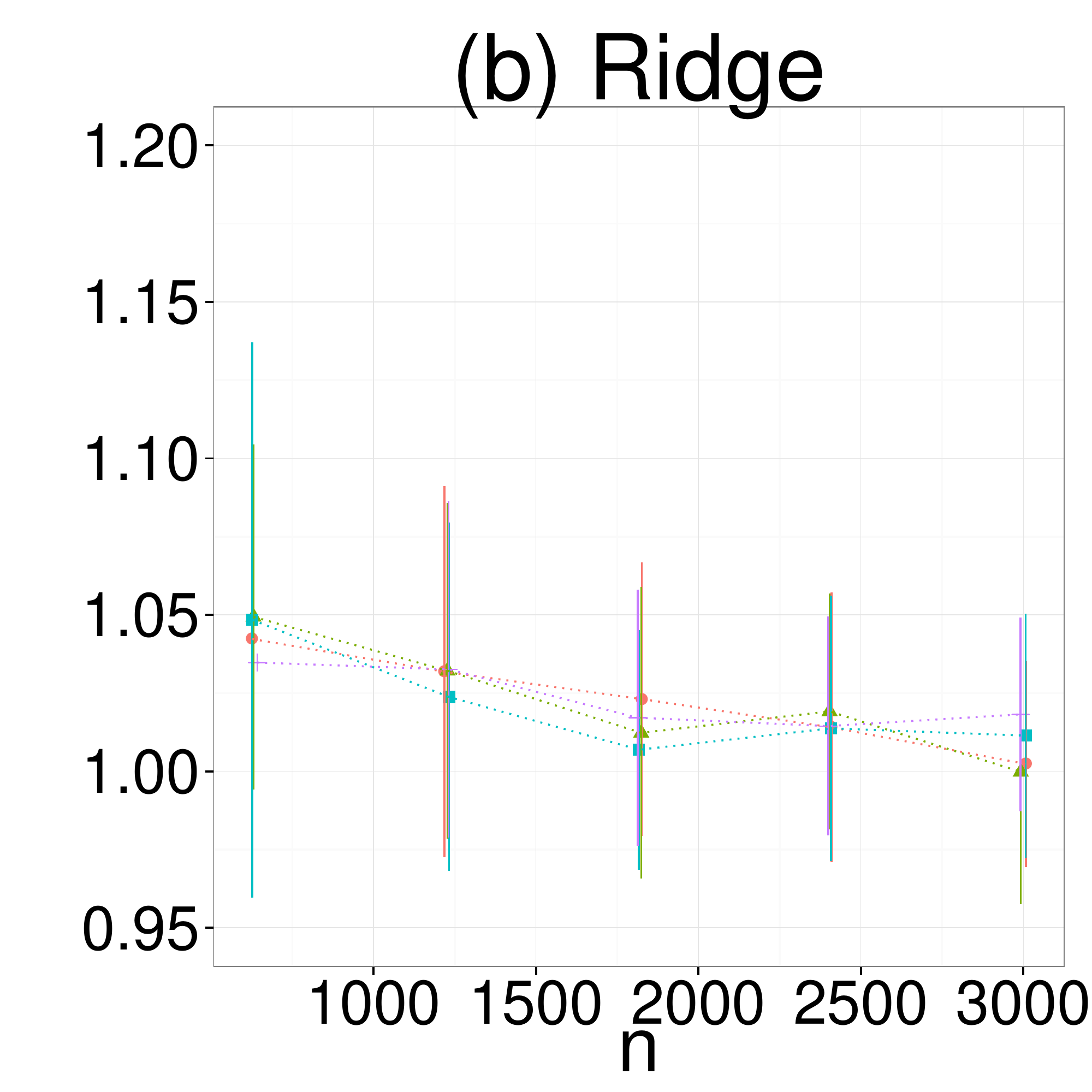}
        \includegraphics[width=0.24\textwidth]{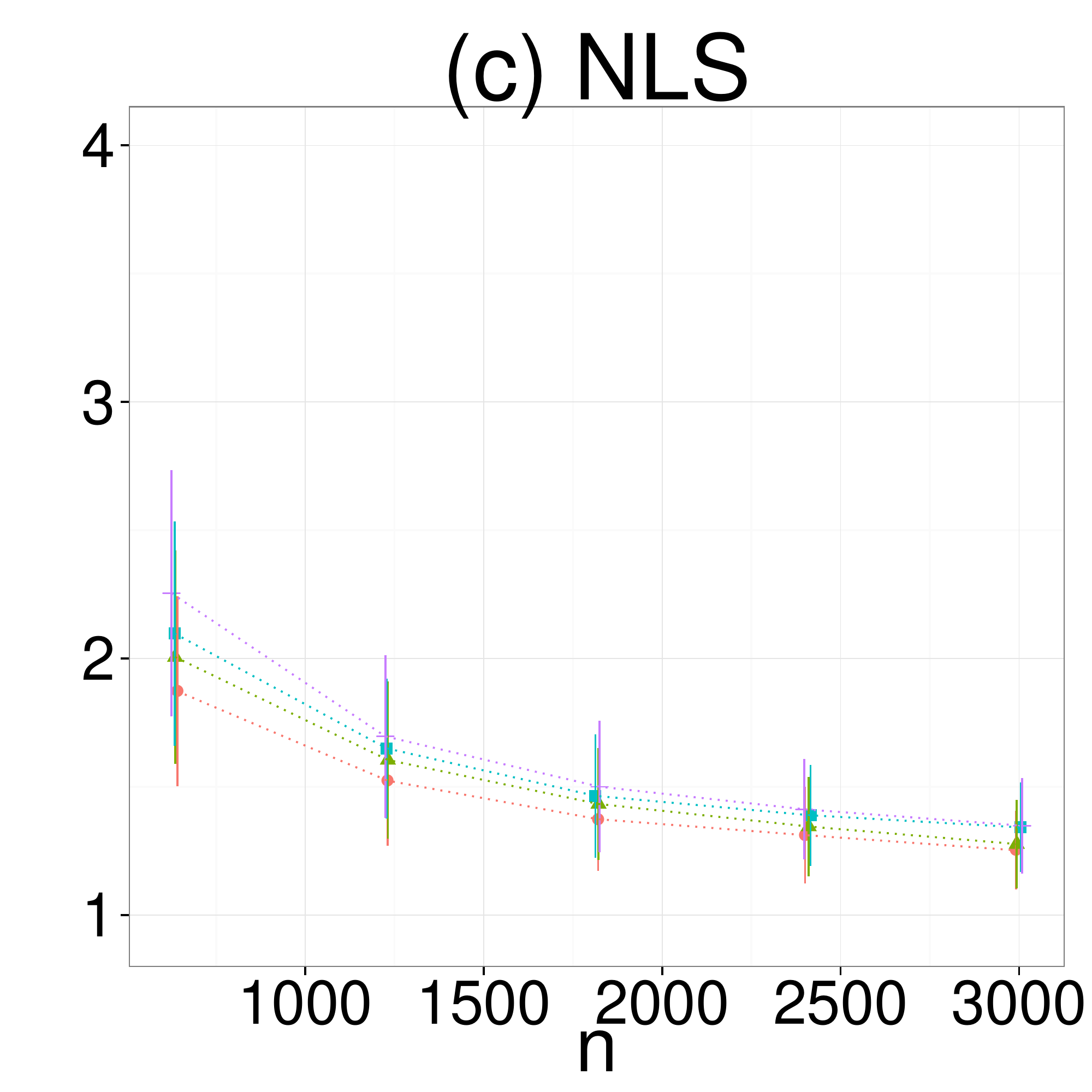}
        \includegraphics[width=0.24\textwidth]{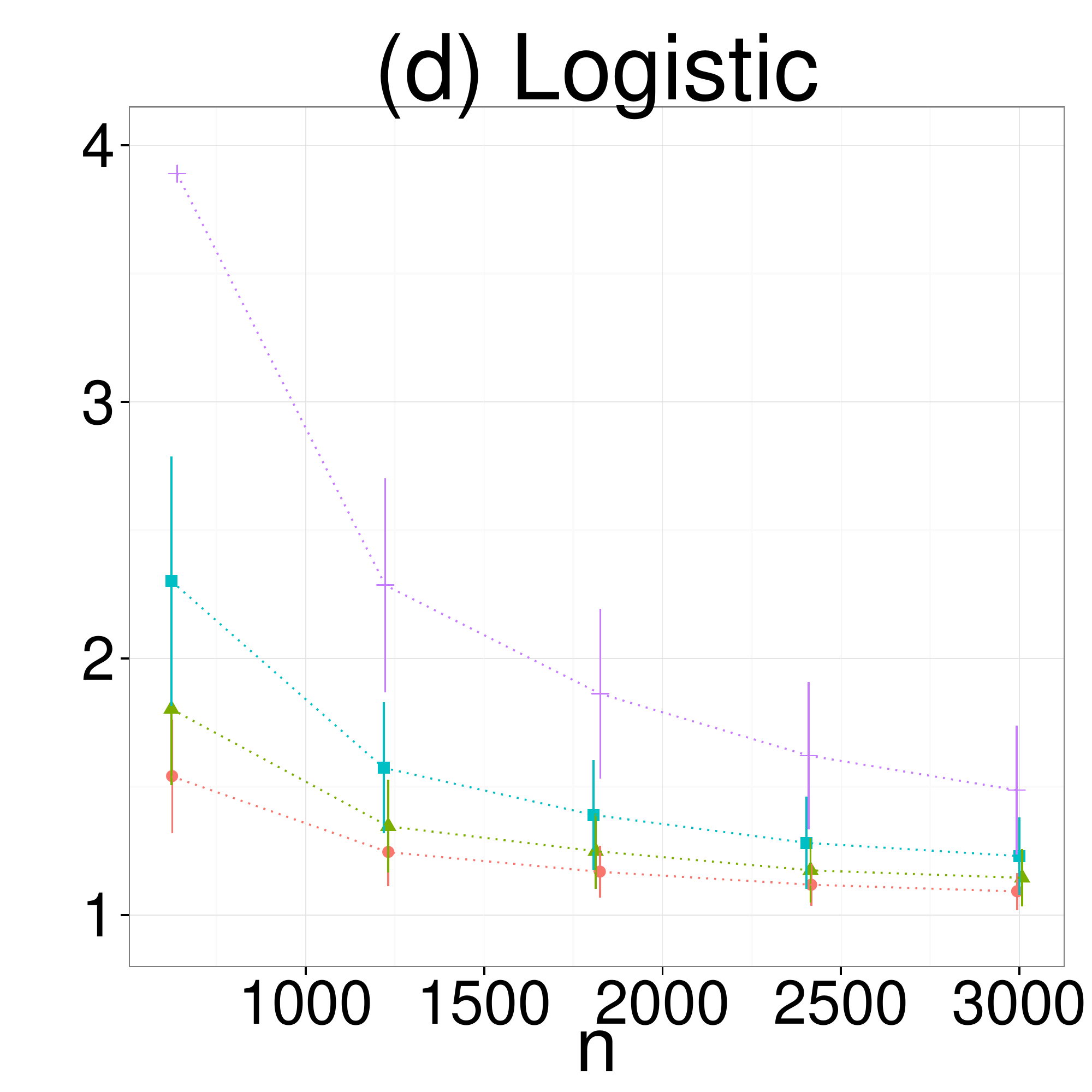}
        \caption{The error ratio $\norm{\parallelized-\theta^*}_2 / \norm{\centralized-\theta^*}_2$ as a function of sample size \(n\) in the fixed-$p$ regime.
          The center point is the median over $500$ replications and the bars represent the median absolute deviation. In all four panels \(p=50\). Color encodes the number of machines $m=10,20,40,100$.
        The learning tasks in the four panels are:\ (a) ordinary least squares; (b) ridge regression; (c) non-linear least squares; (d)\ logistic regression. 
        Data was generated as follows:         
        $X \sim \mathcal{N}(0_{p}, I_{p\times p})$;
        $\theta_0=\tilde{\theta} / \norm{\tilde{\theta}}$, and $\tilde{\theta}_j=j$ for $j=1,\dots,p$;
        $\varepsilon \sim \mathcal{N}(0,10)$;
        In (a)-(b), the response was drawn from $Y = X' \theta_0 + \varepsilon$;
        In (b) $\lambda=0.1$;
        In (c)   $Y = \exp(X' \theta_0) + \varepsilon$, whereas in panel (d),  $P(Y=1|X)=\frac{1}{1+\exp(-X'\theta_0)}$.
        \label{fig:classic_asymptotics_1}
        }
\end{figure}

We start with the fixed-\(p\), large-\(n\) regime. Figure~\ref{fig:classic_asymptotics_1}
shows the empirical median and median absolute deviation of the individual ratios $\|\parallelized-\theta^*\|_2/\|\centralized-\theta^*\|_2$
as a function of sample size \(n\), with \(N=nm\) growing as well. As seen from this figure, and in accord with Theorem \ref{thm:fixed_p_loss}, for large $n$ $\parallelized$ is asymptotically equivalent to $\centralized$ and the error ratio tends to one.
We also see that for small to moderate $n$, parallelization may incur a non-negligible excess error, in particular for non-linear models.

Figure~\ref{fig:error_approximation_OLS} presents the empirical bias and MSE of $\parallelized$ in OLS, as a function of number of machines \(m\) with \(N\) fixed, and compares these to their theoretical approximations from Section~\ref{eg:OLS}. 
In accord with Proposition \ref{prop:OLS_second_moments}, the parallelized OLS estimate shows no excess bias. 
In this OLS case, a high-dimensional approximation of the MSE is identical to the fixed-$p$ in panel (b) so the plot is omitted. We thus conclude that both the fixed-$p$ and the high-dim approximations of the MSE are quite accurate for small $m$ (i.e., large-$n$), but underestimate the error as $p/n$ departs from $0$. 
\begin{figure}[t]
        \centering
\includegraphics[width=0.32\textwidth]{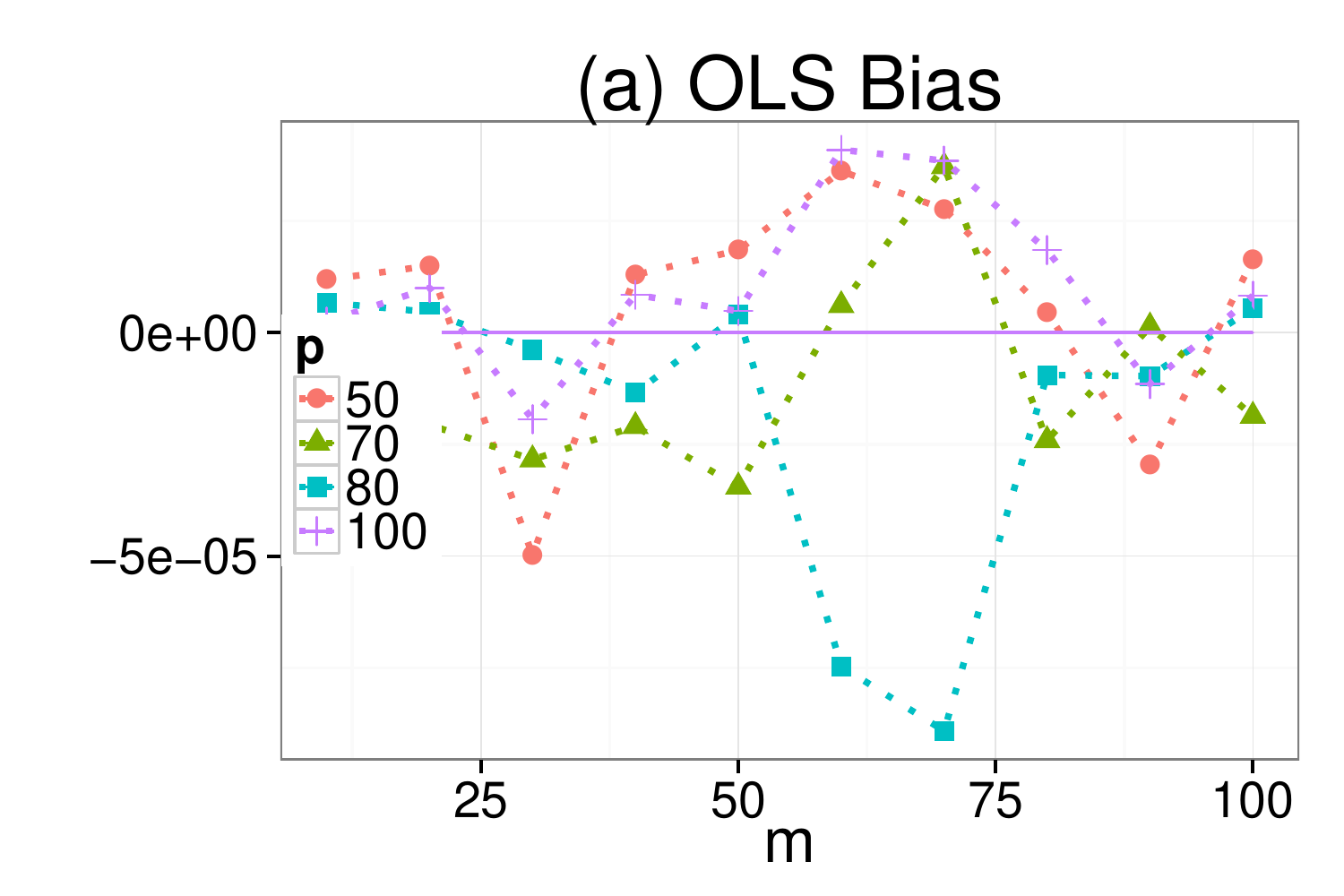}
\includegraphics[width=0.32\textwidth]{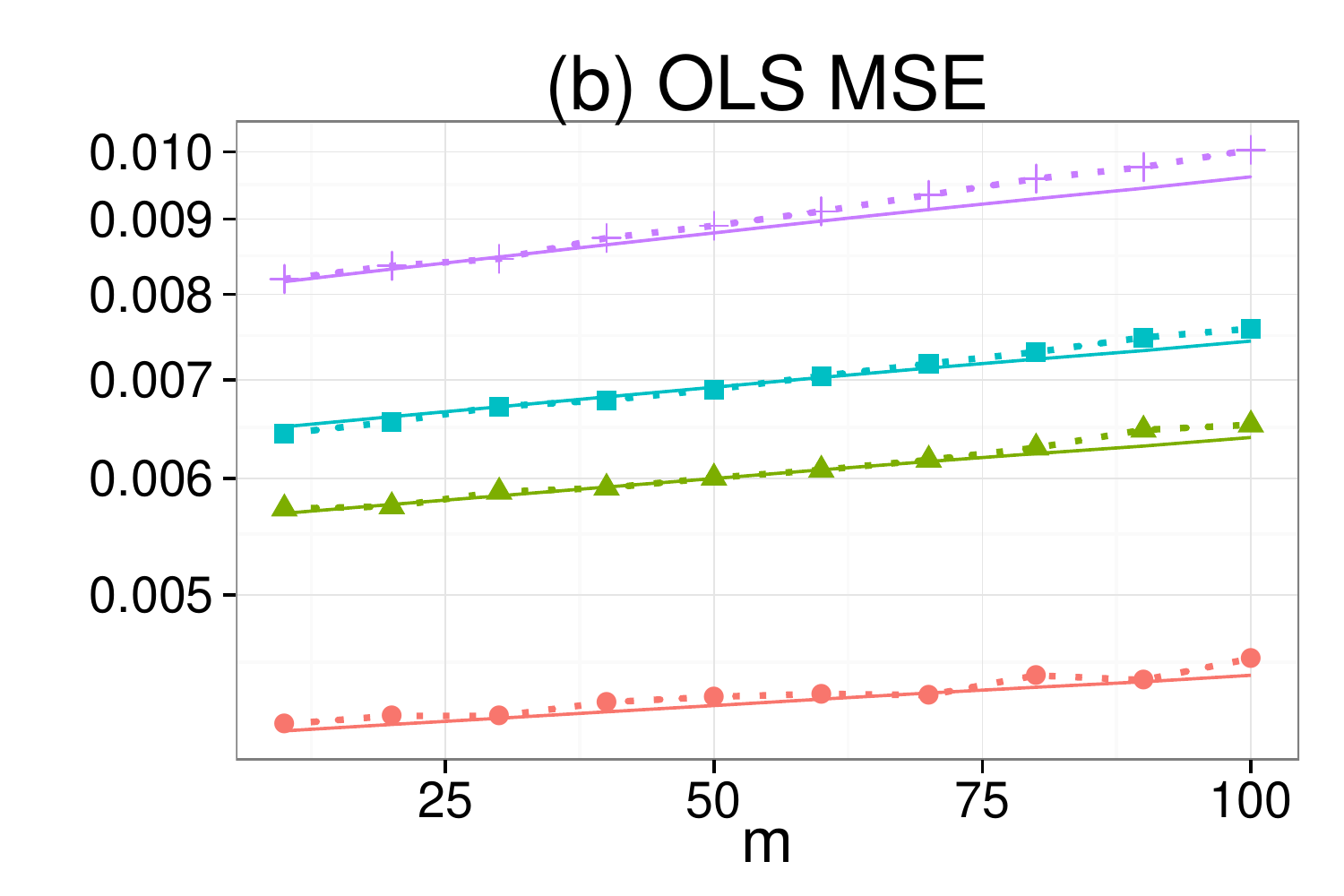}
        \caption{
        Bias and MSE in OLS as a function of number of machines \(m\), with a fixed total number of samples \(N=50,000\), for different dimensions \(p=50,70,80,100\), averaged over $1,000$ replications. 
        Panel (a) shows, in dotted lines, the mean of the empirical bias of an arbitrary coordinate \(j\) in the parallelized estimator $\expect{\parallelized_j-\theta^*_j}$.
        The solid line is the theoretical expression for the second order bias, which, as discussed in Section~\ref{eg:OLS}, is precisely zero. 
        Panel (b) shows, in dotted lines,  the empirical mean squared error of the parallelized estimator, \(\expect{\|\parallelized-\theta^*\|_2^2}\) as a function of \(m\). 
        The solid lines are the theoretical approximation using the second order fixed-$p$ expansion $\Tr(M_2(\bar{\theta}))$, with $M_2(\bar{\theta})$ from Eq.(\ref{eq:second_order_ols}).
        In (b) the $y$-axis is $log_{10}$ scaled.
        Data was generated as follows:         
                $X \sim \mathcal{N}(0_{p}, I_{p\times p})$;
                $\theta_0 = \tilde{\theta} / (\norm{\tilde{\theta}}/10)$ where $\tilde{\theta}_j=j$ for $j=1,\dots,p$.
        $Y = X' \theta_0 + \varepsilon$ where
        $\varepsilon \sim \mathcal{N}(0,2)$.
        \label{fig:error_approximation_OLS}
}
\end{figure}
\begin{figure}[t]
        \centering
\includegraphics[width=0.32\textwidth]{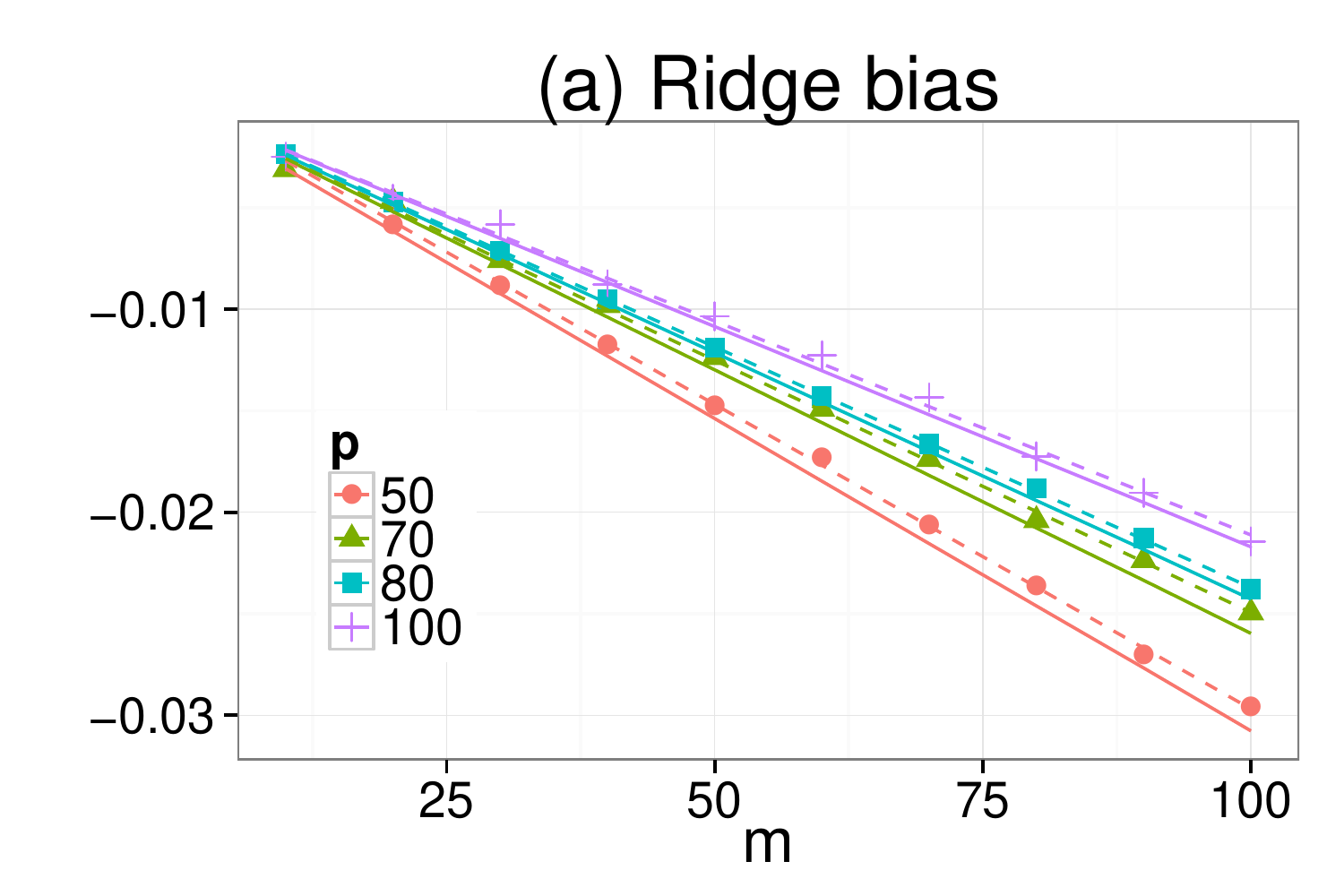}
\includegraphics[width=0.32\textwidth]{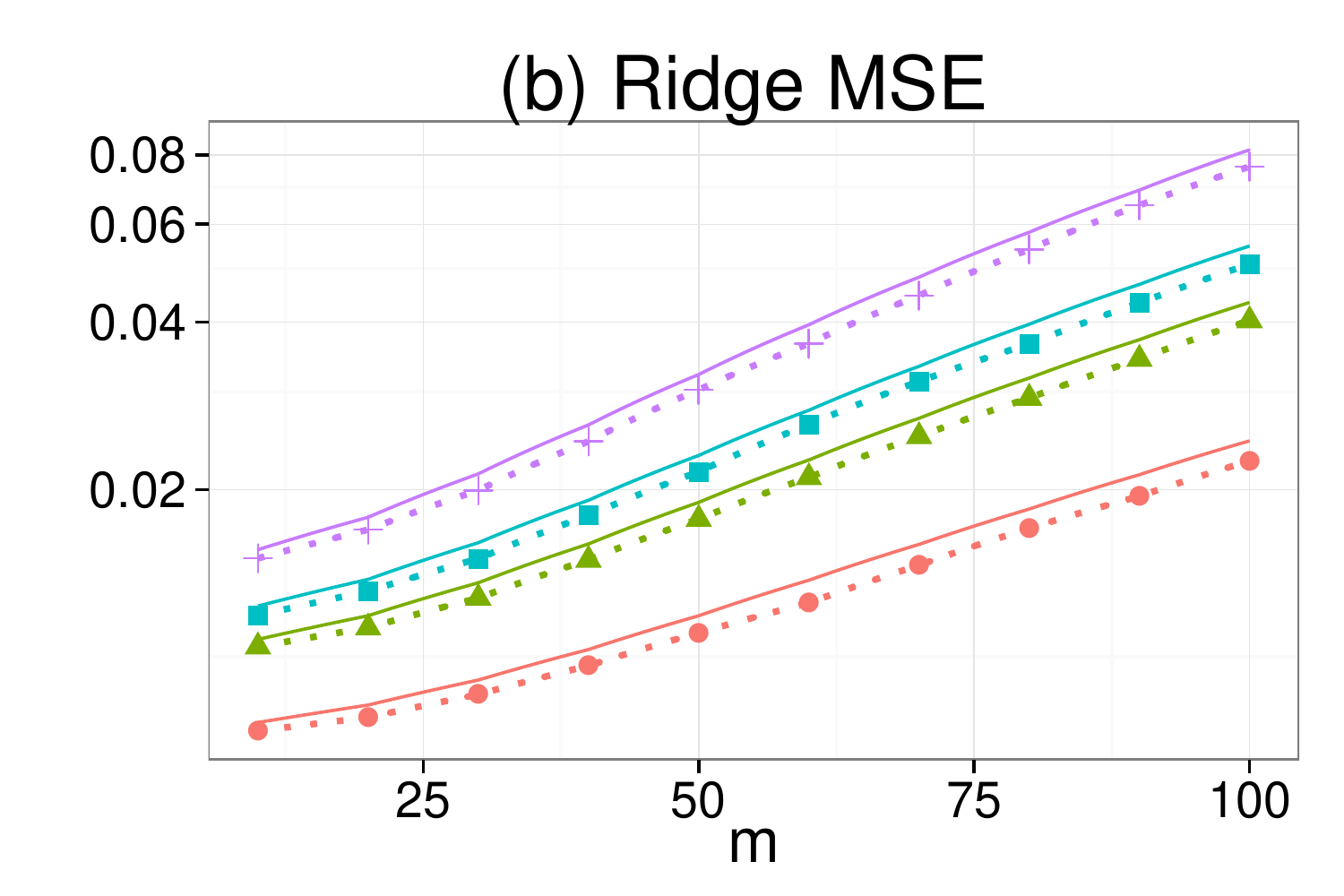}        
\caption{
                Bias and MSE in ridge regression as a function of number of machines \(m\), with a fixed total number of samples \(N=50,000\), for different dimensions \(p=50,70,80,100\), averaged over $1,000$ replications.
                Panel (a) shows, in dotted lines, the mean of the empirical bias of an arbitrary coordinate \(j\) in the parallelized estimator, $\expect{\parallelized_j-\theta_j^*}$.
                The solid line is the theoretical expression for the second order bias from Section~\ref{eg:ridge}. 
                Panel (b) shows, in dotted lines,  the empirical MSE of the parallelized estimator, \(\expect{\|\parallelized-\theta^*\|_2^2}\) as a function of \(m\). 
                The solid lines are the theoretical approximation using a second order fixed-$p$ expansion of the error: $\Tr(M_2(\bar{\theta}))$ where $M_2(\bar{\theta})$ is given in Eq.(\ref{eq:second_order_ridge}).
                In (b) the $y$-axis is $log_{10}$ scaled.
                Data was generated as follows:
                $\lambda$ is fixed at $1$.         
                $X \sim \mathcal{N}(0_{p}, I_{p\times p})$;
                $\theta_0 = \tilde{\theta} / (\norm{\tilde{\theta}}/10)$ where $\tilde{\theta}_j=j$ for $j=1,\dots,p$.
                $Y = X' \theta_0 + \varepsilon$ where
                $\varepsilon \sim \mathcal{N}(0,2)$.
                \label{fig:error_approximation_Ridge}
               }
\end{figure}

Figure~\ref{fig:error_approximation_Ridge} is similar to Figure~\ref{fig:error_approximation_OLS}, but for ridge regression.  We see that, unlike the OLS problem, the ridge problem does have parallelization bias, as predicted by our analysis in Section~\ref{eg:ridge}. 
While our fixed-$p$ MSE approximation is accurate for small $m$ (i.e., large-$n$), for larger \(m\) the empirical error is smaller than that predicted by our second order analysis. This suggests that higher order error terms in the MSE matrix are negative definite (see also Remark~\ref{rem:ridge_anomaly}). 
\begin{figure}[t]
\centering
        \includegraphics[width=0.24\textwidth]{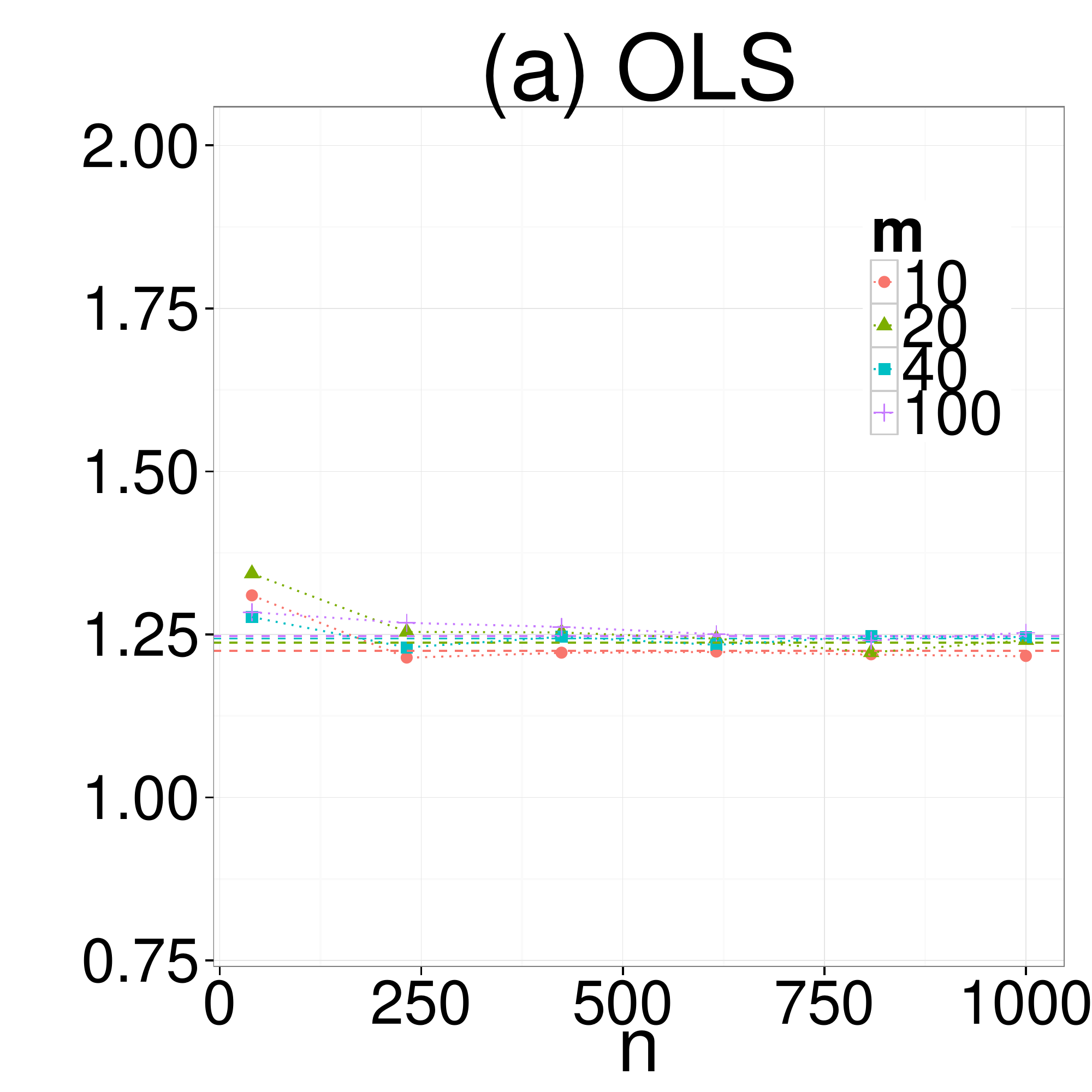}
        \includegraphics[width=0.24\textwidth]{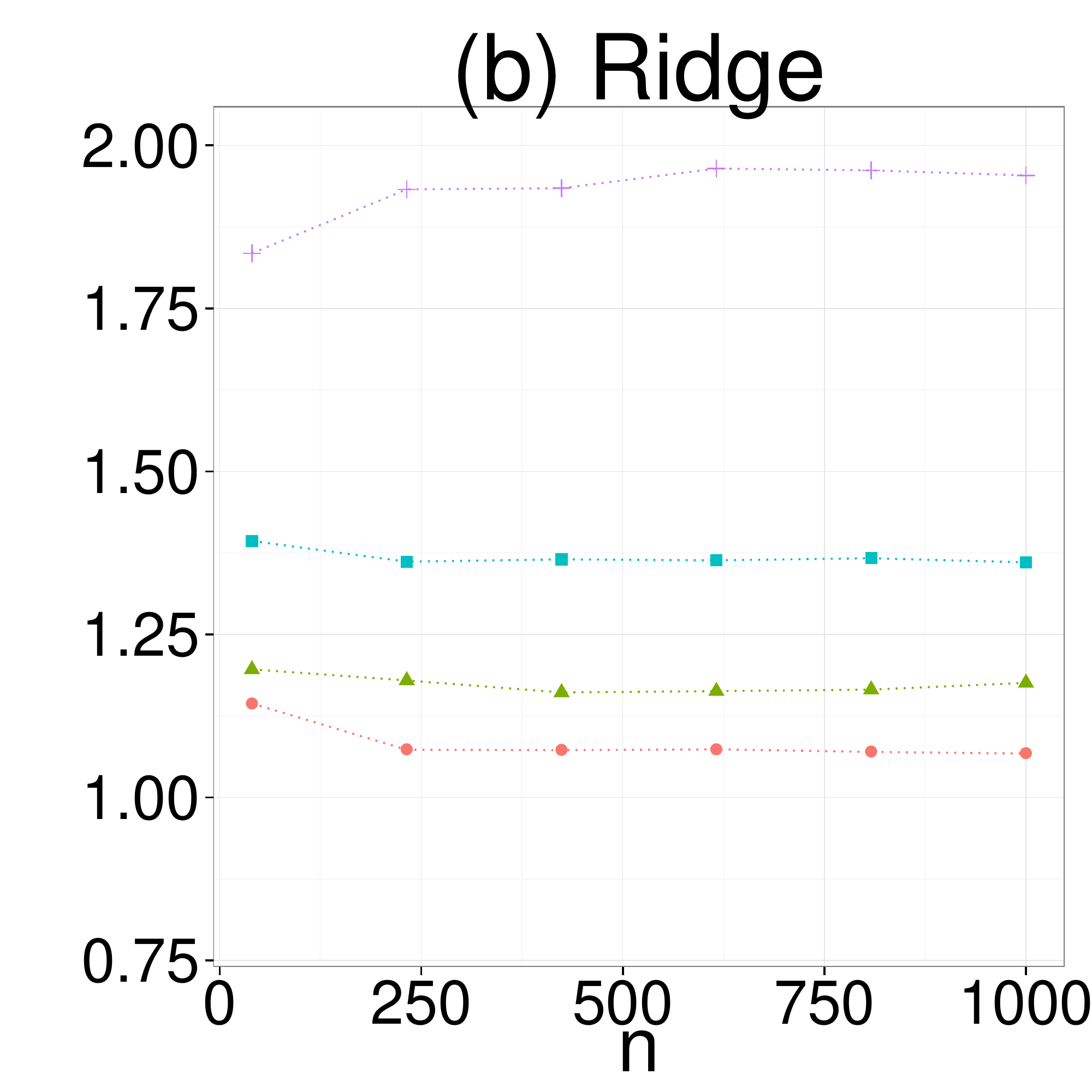}
        \includegraphics[width=0.24\textwidth]{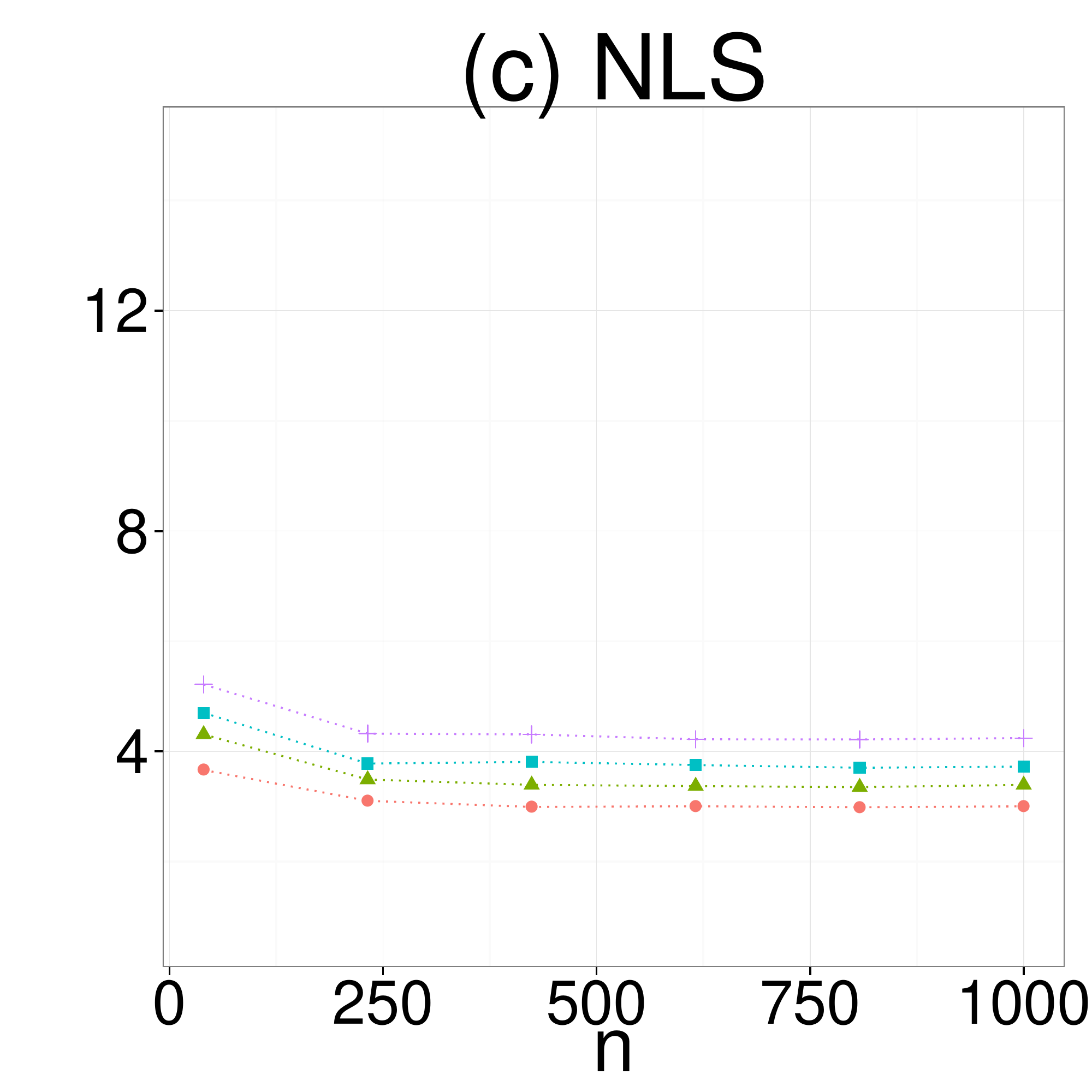}
        \includegraphics[width=0.24\textwidth]{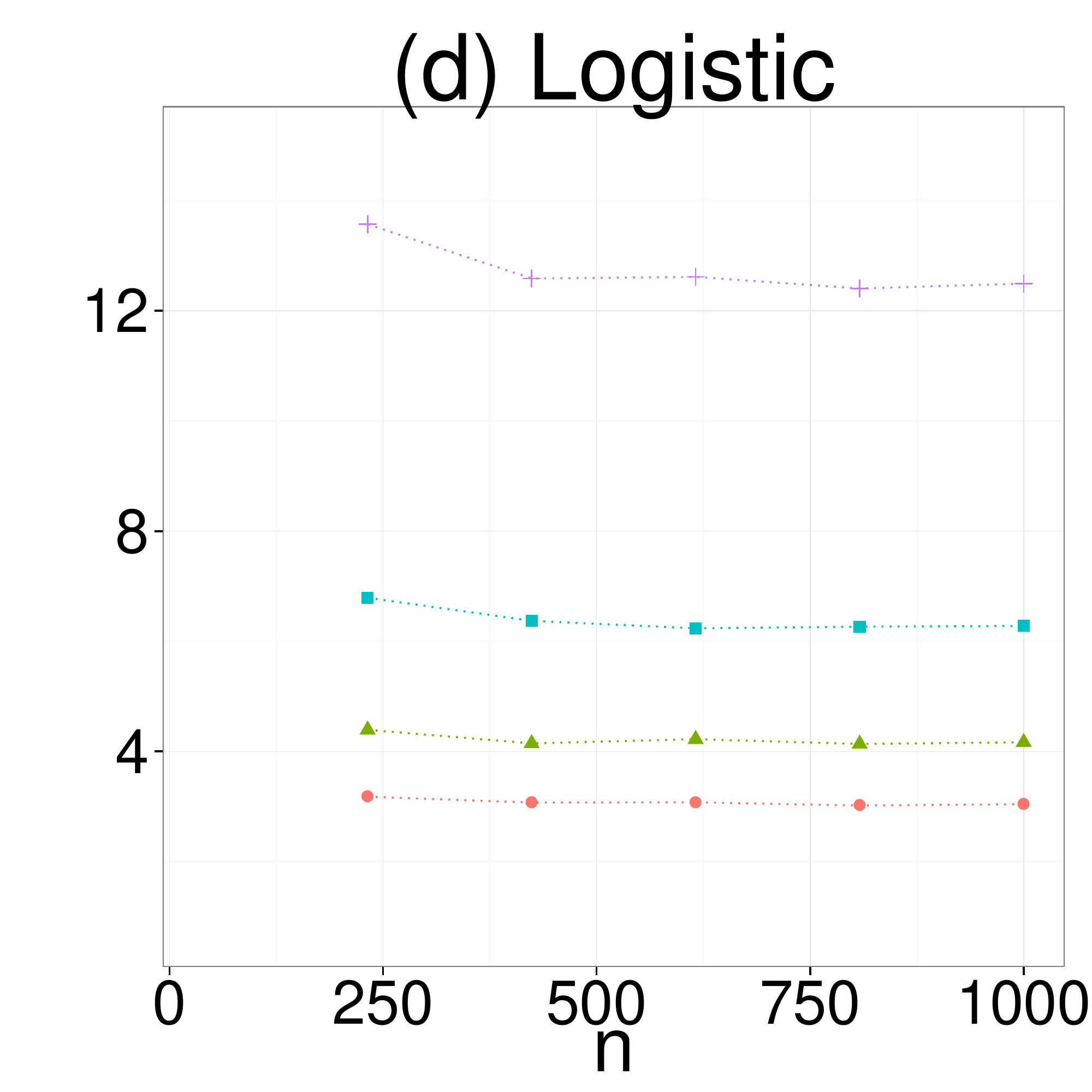}
\caption{
                MSE ratio $\expect{\norm{\bar{\theta}-\theta^*}^2_2}/\expect{\norm{\hat{\theta}_N-\theta^*}^2_2}$ in the high-dimensional regime (with fixed \(\kappa=p/n=0.2\)) 
        as a function of sample size \(n\), averaged over 500 replications.
        Color encodes different number of machines $m=10,20,40,100$.
The four panels depict different learning tasks as in Figure \ref{fig:classic_asymptotics_1}.                Data was generated as follows:         
                $X \sim \mathcal{N}(0_{p}, I_{p\times p})$;
                $\theta_0=\tilde{\theta} / \norm{\tilde{\theta}}$, and $\tilde{\theta}_j=j$ for $j=1,\dots,p$;
        $\varepsilon \sim \mathcal{N}(0,1)$.
                In (a)-(b), $Y = X' \theta_0 + \varepsilon$.
        In (b) $\lambda=1$.
        In (c) $Y = \exp(X' \theta_0) + \varepsilon$, whereas
in (d)  $P(Y=1|X)=\frac{1}{1+\exp(-X'\theta_0)}$.        \label{fig:highDim_asymptotics_1}
        }
\end{figure}

Next, we consider the high-dimensional regime.
Figure~\ref{fig:highDim_asymptotics_1} shows  $\expect{\norm{\bar{\theta}-\theta^*}^2}/\expect{\norm{\hat{\theta}_N-\theta^*}^2}$  as a function of machine-wise sample size \(n\), while holding $\kappa$ and $m$ fixed.
In contrast to the fixed-$p$ regime, here there is a first-order accuracy loss, and even for large \(n\) the MSE ratio does not converge to one.  
In the OLS case, where our high-dimensional approximations are applicable, they are indeed accurate over a wide range of values of $n$ and $m$. 
As already observed in the fixed-$p$ regime, non-linear models (panels c and d) incur a considerable parallelization excess error. 


\section{Practical Considerations}
\label{sec:how_many_machines}

Parallelization is not necessarily the preferred approach to deal with massive datasets. 
In principle, when $N \gg p$ an easy, though potentially not sufficiently accurate solution, is to discard observations by randomly subsampling the data. 
Parallelization should thus be considered when the accuracy attainable by subsampling is not satisfactory. 
An important question is then over how many machines should the practitioner distribute the data? 
When tackling this question, we distinguish between two scaling regimes: $N$ fixed or $n$ fixed.
Fixed $n$ captures the single-machine storage constraint: the total available data is virtually infinite and using more machines allows processing of more data, and hence better accuracy, at an obvious financial cost. 
Fixed $N$ captures either sampling or computational constraints: here, the total sample size $N$ is fixed and processing it on a single machine might be too slow. 
Thus, splitting the data reduces run-time but also decreases the accuracy. In other words, {\em by parallelizing, we trade accuracy for speed}.
Interestingly, when the number of samples $N$ is fixed, by using our approximations and varying $m$, we are able to trace the accuracy-complexity tradeoff facing the practitioner. 
An informed choice of $m$ is thus choosing either a desirable run-time, or a desired error level, on this curve.

We now formulate the target functions for choosing the number of machines in these two regimes. 
For fixed $n$, wishing to minimize costs, we analyze what is the \emph{minimal} number of machines that attains a desired accuracy, $\error(m)$:
\begin{align}
\label{eq:choose_m_fix_n}
        \min \{ m \quad s.t. \quad \error(m) \leq \epsilon, \text{n samples per machine} \}.
\end{align} 
For fixed $N$, wishing to minimize runtime, and in the spirit of \citet{shalev-shwartz_svm_2008}, we ask what is the \emph{maximal} number of machines so that runtime is minimized while a desired level of accuracy is maintained. 
Choosing the number of machines in the fixed $n$ scenario reduces to solving 
\begin{align}
\label{eq:choose_m_fix_N_m}
        \max \{ m \quad s.t. \quad \error(m) \leq \epsilon , N/m \text{ samples per machine} \}.
\end{align}

Next, let us study these two optimization problems, Eqs.(\ref{eq:choose_m_fix_n}) and (\ref{eq:choose_m_fix_N_m}), when the accuracy measure is $\error(m):=\expect{\norm{\parallelized-\theta^*}^2}$.
This is challenging or even infeasible, since in general we do not have explicit expressions for this quantity. 
Moreover, in the fixed-$p$ regime, approximating the MSE by the asymptotic leading error term yields that this quantity is independent of $m$! 
As we show below, meaningful and interesting solutions to this optimization problems arise when we approximate $\error(m)$ by the second order expression in the fixed-$p$ regime. 
Specifically, using Eq.(\ref{eq:second_order_MSE}) we approximate $\error(m):= \Tr(\MSESecond{\parallelized})$.
In the high-dimensional regime, in contrast, the optimization problems (\ref{eq:choose_m_fix_n}) and (\ref{eq:choose_m_fix_N_m}) are well posed already when we approximate the MSE by the first order term. 
Relying on Eq.(\ref{eq:parallel_highDim_MSE}) gives $\error(m):=  r^2(\kappa)/m  \: \mathbb{E}[\|\Sigma^{-1/2}\xi\|^2_2]$.

We now present the optimization problems corresponding to the fixed-$p$ approximation in each scaling scenario:
\begin{description}
\item[Fixed-n:] 
$
        \min \{ m \quad s.t. \quad      
        \frac{m-1}{m} \frac{1}{n^2} \Tr(\gamma_0)+
        \frac{1}{mn} \Tr(\gamma_1) + 
        \frac{1}{m n^2} \Tr\left(\gamma_2+\gamma_2'+\gamma_3+\gamma_4+\gamma_4' \right)
 \leq \epsilon \},
$ 
which stems from Eq.(\ref{eq:choose_m_fix_n}) and Eq.(\ref{eq:second_order_MSE}).

\item[Fixed-N:] 
$
        \max\{ m \quad s.t. \quad 
        \frac{(m-1)m}{N^2}  \Tr(\gamma_0) +
        \frac{1}{N} \Tr(\gamma_1) + 
        \frac{m}{N^2} \Tr\left(\gamma_2+\gamma_2'+\gamma_3+\gamma_4+\gamma_4' \right)
 \leq \epsilon \},
$
which stems from Eq.(\ref{eq:choose_m_fix_N_m}) and Eq.(\ref{eq:second_order_MSE}) with $n=N/m$.

%
\end{description}

Let us illustrate these formulas in the OLS example from Section~\ref{eg:OLS}. 
The required quantities for OLS are collected in Appendix~\ref{apx:proof_OLS_second_moments}.
For example, solving the fixed-$N$ problem, the maximal number of machines that will keep the \emph{per-coordinate} MSE under $0.2$, i.e., $\varepsilon p=0.2$, with $N=10^6$, $p=100$, and $\sigma^2=10$ is $m\leq 9,901$.
Alternatively, assuming an abundance of data and a memory limit such that $n\leq10^4$, we solve the fixed-$n$ problem to find that $m \geq 51$ will satisfy the derived error level.

\comment{\paragraph{Choosing $m$ for OLS}
\begin{description}

\item[Fixed-p--fixed-n:] 
$
        \min \{ m \quad s.t. \quad 
        \frac{\sigma^2 \Tr(\Sigma^{-1})}{m}     ( n^{-1} + n^{-2} (1+p) )       
 \leq \epsilon \}.
$ 

\item[Fixed-p--fixed-N:] 
$
        \max\{ m \quad s.t. \quad 
        \frac{1}{N} \sigma^2 \Tr(\Sigma^{-1})  + \frac{m}{N^2}  (1+p) \sigma^2 \Tr(\Sigma^{-1})
        \leq \epsilon \}.
$

\item[High-dim--fixed-n:] 
$
                \min \{ m \quad s.t. \quad 
                \frac{\kappa}{m(1-\kappa)} \sigma^2 \: \mathbb{E}[\|\Sigma^{-1/2}\xi\|^2_2]  
                \leq \epsilon \}.
$

\item[High-dim--fixed-N:] 
$       
        \max \{ 
                m \quad s.t. \quad 
                \frac{p}{N-pm} \sigma^2 \: \mathbb{E}[\|\Sigma^{-1/2}\xi\|^2_2]
                \leq \epsilon \,  
        \}.
$
\end{description}
}

\begin{remark}
                In some cases, the practitioner may wish to control the parallelization error relative to the centralized solution, and not as an absolute value as analyzed above. Namely, the restriction is now $\error(m) \leq (1+\varepsilon) \error(1)$.
                The scenarios (fixed $n$/$N$) and approximations previously discussed apply here as well.
                For example, in our OLS example, solving the Fixed-$N$ problem with $N=10^6$, $p=100$, and $\sigma^2=10$, 
                yields that $m \leq 991$ for $\parallelized$ to err no more than $10\%$ more than $\centralized$ ($\varepsilon=0.1$).
                On the other hand, For the Fixed-$n$ problem, with $n=10^4$, $p=100$, and $\sigma^2=10$, we can parallelize up to $m\leq 102$ machines, and still maintain the same $10\%$ excess error allowance.
\end{remark}


\section{Discussion}
\label{sec:discussion}

In this work we studied the error of parallelized M-estimators when $N$ observations are uniformly at random distributed over $m$ machines. Each machine then learns a $p$ dimensional model with its $n=N/m$ observations and the \(m \) machine-wise results are averaged to a global estimate $\parallelized$.
We derived several different approximations of the estimation error in $\parallelized$ with different quantitative and qualitative insights.

\paragraph{Insights}

When $n\gg p$ not much accuracy is lost by splitting the data.
This stands in contrast to other works that demonstrate how, under different assumptions, parallelization combined with averaging may incur a large error \citep{liu_distributed_2014_2}, or even an unbounded one \citep{shamir_communication_2013}.
Our analysis can thus be viewed as providing sufficient conditions for parallelization to be a suitable approach to reduce the overall run-time. 
A second insight is that if the model is highly non-linear, then the excess paralellization error may be considerably large. 

In contrast to the classical fixed-$p$ regime, our high-dimensional analysis, currently confined to generative linear models, showed that splitting the data when there are only few observations per parameter always takes its accuracy toll.
The degradation in accuracy due to splitting can still be quantified even though estimates converge to non-degenerate random limits.

\paragraph{Future Research}
At the basis of our work is an attempt to adhere to real-life software and hardware constraints of parallelized learning.
The assumption of uniform and random distribution of samples to machines is  realistic for some applications, and certainly facilitates the mathematical analysis. It may also be overly restrictive for other appications.
A venue for future research is thus the relaxation of this assumption, allowing for some systematic difference between machines.
We also aim at analyzing other aggregation schemes. Particularly ones that employ more than the mere machine-wise point estimate, and apply to non convex parameter spaces. 
An example of such is the Minimum Kullback-Leibler divergence aggregation, proposed by \citet{liu_distributed_2014_2}. 
This may extend the applicability of our results, for example, to image, sound, and graph data.

\section*{Acknowledgments}
We thank Derek Bean, Kyoo il Kim, Yaakov Ritov, Saharon Rosset,  Ohad Shamir and  Yuchen Zhang for fruitful discussions. This research was partly supported by a grant from the Intel Collaborative Research Institute for
Computational Intelligence (ICRI-CI).


\appendix
\numberwithin{equation}{section}
\counterwithin{theorem}{section}

\section{Proof of Theorem~\ref{thm:fixed_p_loss}}
\label{apx:proof_fixed_p}

Under Assumption Set \ref{asum:first_order}, classical statistical theory guarantees that upon optimizing the empirical risk (\ref{eq:R_N}), the resulting estimators, $\machine{j}$, converge in probability to $\theta^*$ at rate $n^{-1/2}$.  
Moreover, the leading error term is linear in the influence functions $\nabla f(Z_i,\theta^*)$~\cite[Theorem 5.23]{vaart_asymptotic_1998}:
\begin{align}
\label{eq:asympt_linear}
        \machine{j} &= \theta^* - \invInf \nabla \hat R_n^j(\theta^*)  + o_P(n^{-1/2}) \\
        &= \theta^* - \invInf \frac{1}{n} \sum_{i\in [j]} \nabla f(Z_i,\theta^*)  + o_P(n^{-1/2}). \nonumber
\end{align}
where $[j]$ denotes the $n$ indexes of the observations assigned to machine $j$.
Taking the average of the machine-wise estimators over a fixed number of machines $\parallelized := \frac{1}{m} \sum_j \machine{j}$, and applying Eq.(\ref{eq:asympt_linear}) yields
\begin{align}
\label{eq:linear_representation}
        \parallelized &= \theta^* - \invInf \nablajr \hat R_N(\theta^*)  +  \; o_P(n^{-1/2})  \\
         &= \theta^* - \invInf \frac{1}{N} \sum_{i=1}^N  \nabla f(Z_i,\theta^*)  + o_P(n^{-1/2}). \nonumber
\end{align}
Similarly, applying Eq.(\ref{eq:asympt_linear}) to the centralized solution:
\begin{align*}
        \hat\theta_N &= \theta^* - \invInf \nablajr \hat R_N(\theta^*)  + o_P(N^{-1/2}) \;.   
\end{align*}
Since $m$ is fixed, $o_P(N^{-1/2})=o_P(n^{-1/2})$. Eq.(\ref{eq:unimprovable}) now follows.\hfill$\Box$


\section{Proof of Theorem~\ref{thm:second_order_bias}}
\label{apx:proof_fixed_p_second_order_bias}

Under Assumption Set ~\ref{assum:second_order}, with $s=3$, by Proposition 3.2 in \citep{rilstone_second-order_1996} $\machinewise$ admits the expansion $\theta^* + \errorSecond{\machinewise}+ O_P(n^{-3/2})$.
We can thus decompose
\begin{align*}
        \biasSecond{\parallelized} &= 
        \expect{\frac{1}{m}\sum_j \errorFirst{\machine{j}} +\frac{1}{m}\sum_j \xi_{-1}(\machine{j})}.
\end{align*}
By definition, $\errorFirst{\machinewise} = - \invInf \nabla\hat R_n(\theta^*)$. 
Hence, $\expect{\errorFirst{\machine{j}}}=0$, for all $j$ and the first term in the equation above vanishes. As for the second term, clearly  $\expect{\xi_{-1}(\machine{j})}$ is independent of $j$, thus, 
$
        \biasSecond{\parallelized}=\expect{\xi_{-1}(\machinewise)}
$.

Again, according to \citep[Proposition 3.2]{rilstone_second-order_1996},  $\expect{\xi_{-1}(\machinewise)}=\delta/n$, where $\delta$ depends on the various problem parameters, but is independent of $n$. Since $\biasSecond{\centralized}=\expect{\xi_{-1}(\centralized)}=\delta/N$, Eq.(\ref{eq:second_order_loss}) readily follows. 
\hfill$\Box$

\section{Proof of Theorem~\ref{thm:second_order_MSE}}
\label{apx:proof_fixed_p_second_order_MSE}

Under Assumption Set ~\ref{assum:second_order}, with $s=4$, we can expand $\parallelized$ as in Eq.(\ref{eq:second_order}).
Plugging this asymptotic expansion of $\parallelized$ into the definition of $\MSESecond{\parallelized}$ from Eq.(\ref{eq:second_order_MSE_definition}), and collecting terms up to $O(n^{-2})$ we have
\begin{align*}
        \MSESecond{\parallelized} =& 
         \expect{\xi_{-1/2}(\parallelized) \xi'_{-1/2}(\parallelized)} +
         \expect{\xi_{-1}(\parallelized) \xi'_{-1/2}(\parallelized)} + 
         \expect{\xi_{-1/2}(\parallelized) \xi'_{-1}(\parallelized)} + \\
         & \expect{\xi_{-1}(\parallelized) \xi'_{-1}(\parallelized)} + 
         \expect{\xi_{-3/2}(\parallelized) \xi'_{-1/2}(\parallelized)} + 
         \expect{\xi_{-1/2}(\parallelized) \xi'_{-3/2}(\parallelized)}. 
\end{align*}
We now analyze each summand separately.
$
        \expect{\xi_{-1/2}(\parallelized) \xi'_{-1/2}(\parallelized)} =
        \frac{1}{m^2} \sum_{k,l} \expect{\xi_{-1/2}(\machine{k}) \xi'_{-1/2}(\machine{l})}, 
$
for $k \neq l$ then $\expect{\xi_{-1/2}(\machine{k}) \xi'_{-1/2}(\machine{l})}$ vanishes.
For 
$k=l$ we have $m$ terms which equals $n^{-1} \gamma_1$ by definition.
The same analysis holds for 
$ \expect{\xi_{-1}(\parallelized) \xi'_{-1/2}(\parallelized)}$ and 
$ \expect{\xi_{-3/2}(\parallelized) \xi'_{-1/2}(\parallelized)}$ denoted  
$n^{-2} \gamma_2$ and $n^{-2} \gamma_4$ respectively. 
As for 
$
        \expect{\xi_{-1}(\parallelized) \xi'_{-1}(\parallelized)} =
        \frac{1}{m^2} \sum_{k,l} \expect{\xi_{-1}(\machine{k}) \xi'_{-1}(\machine{l})} :
$
for $k \neq l$ we have $m(m-1)$ terms where
$
        \expect{\xi_{-1}(\machine{k}) \xi'_{-1}(\machine{l})} =
        \expect{\xi_{-1}(\machine{k})}  \expect{\xi'_{-1}(\machine{l})} 
$
which we defined as $n^{-2} \gamma_0$.
For the remaining $m$ terms where $k = l$, then 
$
        \expect{\xi_{-1}(\machine{k}) \xi'_{-1}(\machine{l})} = n^{-2} \gamma_3
$
by definition.
Collecting terms completes the proof.

\hfill$\Box$

\section{Proof of Proposition~\ref{prop:OLS_second_moments}}
\label{apx:proof_OLS_second_moments}

Denoting
$
A^{-k}:=(A^{-1})^k \;;
S_i := X_iX_i' \;;
\dot{f}_i:=\nabla f(X_i,\theta^*) \;;
\ddot{f}_i:=\nabla^2 f(X_i,\theta^*) \;;
V_i:= (\ddot{f}_i - \Inf)  \;;
d_i:= \invInf \dot{f}_i \;;
W_\theta = \nabla^3 R(\theta).
$
We also denote by $W_\theta(x,y)$ the linear operator in $\mathbb{R}^p$ returned by evaluating $W_\theta $ at $(x,y) \in \mathbb{R}^p \times \mathbb{R}^p$.

In our OLS setup we have:
$       \nabla R(\theta) = -  \Sigma (\theta_0-\theta);\;
\invInf = \Sigma^{-1} ;\: 
W_\theta \equiv 0;\: 
\dot{f}_i = -X_i \varepsilon_i ;\:
\ddot{f}_i = S_i ;\:
V_i = S_i-\Sigma ;\:
d_i = - \Sigma^{-1} X_i \varepsilon_i
$.      
From the proof of Proposition $3.2$ in \citep{rilstone_second-order_1996} we have
\begin{align}
\label{eq:expected_second_error_term}
        \expect{\xi_{-1}(\machinewise)} = n^{-1} \invInf (\expect{V_1 d_1} - \frac{1}{2} \expect{\W{d_1,d_1}}).
\end{align}
As $W_\theta \equiv 0   $ and $\expect{\varepsilon_1}=0$, independent of $X_1$, then $\expect{\xi_{-1}(\machinewise)}=0$,
so that from Eq.(\ref{eq:gammas}), $\gamma_0=0$.

Next, we turn to the matrix $\gamma_1$. 
From \citep[page 374]{rilstone_second-order_1996} we have
\begin{align*}
        \gamma_1 = \invInf \expect{\dot{f}_1 \dot{f}'_1} \invInf.
\end{align*}
Since $\dot{f}_1 = -X_1 \varepsilon_1$, and $X_1$ independent of $\varepsilon_1$, then $\expect{\dot{f}_1 \dot{f}'_1}= \sigma^2 \Sigma$. 
Now recalling that $\Inf=\Sigma$ we obtain the form of $\gamma_1$.

Next, we analyze the matrix $\gamma_2$.
From \citep[Lemma 3.3]{rilstone_second-order_1996} we have
\begin{align*}
        \gamma_2= \invInf \left(
                -\expect{V_1 d_1 d'_1} + \frac{1}{2} \expect{\W{d_1,d_1}d'_1}
        \right).
\end{align*}
The second summand vanishes and
\begin{align*}
        \expect{V_1 d_1 d'_1} &= \expect{(S_1-\Sigma)(\Sigma^{-1} X_1 \varepsilon_1)(\Sigma^{-1} X_1 \varepsilon_1)'} \\
        &= \sigma^2 \expect{S_1 \Sigma^{-1} S_1 \Sigma^{-1} - \Sigma \Sigma^{-1} S_1 \Sigma^{-1}} 
        = \sigma^2 \expect{S_1 \Sigma^{-1} S_1 \Sigma^{-1} -  S_1 \Sigma^{-1}} .
\end{align*}
As $S_1 \sim Wishart_p(1, \Sigma)$, we call upon Wishart theory.
We collected the required properties in Appendix~\ref{apx:wishart_matrices}. 
Applying Theorem~\ref{thm:wishart_one} to each summand we have that 
$\expect{S_1 \Sigma^{-1}} = I$, and 
$\expect{S_1 \Sigma^{-1} S_1 \Sigma^{-1} } = (2 \Sigma \Sigma^{-1} \Sigma + \Tr(\Sigma \Sigma^{-1})\Sigma) \Sigma^{-1} = (2+p) I$.

As for the matrix $\gamma_3$, \citet[Lemma 3.3]{rilstone_second-order_1996} yields
\begin{align*}
        \gamma_3= \invInf \left(
                \expect{V_1 d_1 d'_2 V'_2} +
                \expect{V_1 d_2 d'_1 V'_2} +
                \expect{V_1 d_2 d'_2 V'_1}
        \right) \invInf  .
\end{align*}
Because $\expect{\varepsilon_1 \varepsilon_2}=0$ the first two terms cancel.
We now compute $\expect{V_1 d_2 d'_2 V'_1}$.
Recalling that $S_1$ is independent of $S_2$
\begin{align*}
        \expect{V_1 d_2 d'_2 V'_1} &= \expect{(S_1-\Sigma)(\Sigma^{-1} X_2 \varepsilon_2^2  X_2' \Sigma^{-1})(S_1-\Sigma)} \\
        &= \sigma^2 \expect{
                S_1 \Sigma^{-1} S_2 \Sigma^{-1} S_1 -
                S_1 \Sigma^{-1} S_2 \Sigma^{-1} \Sigma  -
                \Sigma \Sigma^{-1} S_2 \Sigma^{-1} S_1 +
                \Sigma \Sigma^{-1} S_2 \Sigma^{-1} \Sigma
                 }\\
         &= \sigma^2 \left(
                 \expect{S_1 \Sigma^{-1} S_2 \Sigma^{-1} S_1} -
                 \Sigma - \Sigma + \Sigma \right).
\end{align*}
Applying the results in Theorem~\ref{thm:wishart_one}, then $\expect{S_1 \Sigma^{-1} S_2 \Sigma^{-1} S_1 } = (p+2) \Sigma$.

Finally, we study the matrix $\gamma_4$.
Since $W_\theta$ and higher derivatives, all vanish, then \citep[Lemma 3.3]{rilstone_second-order_1996} yields
\begin{align*}
        \gamma_4= \invInf \left(
                \expect{V_1 \invInf V_1 d_2 d'_2 } +
                \expect{V_1 \invInf V_2 d_1 d'_2 } +
                \expect{V_1 \invInf V_2 d_2 d'_1 } 
        \right) .
\end{align*}
Because $\expect{\varepsilon_1 \varepsilon_2}=0$ the last two terms cancel.
We now compute $\expect{-V_1 \invInf V_1 d_2 d'_2}$:
\begin{align*}
        \expect{V_1 \invInf V_1 d_2 d'_2} &= 
        \expect{(S_1-\Sigma) \Sigma^{-1} (S_1-\Sigma) (\Sigma^{-1}  X_2 \varepsilon_2^2  X_2' \Sigma^{-1})} \\
        &= \sigma^2 \expect{
                S_1 \Sigma^{-1} S_1 \Sigma^{-1} S_2 \Sigma^{-1} -
                \Sigma \Sigma^{-1} S_1 \Sigma^{-1} S_2 \Sigma^{-1} -
                S_1 \Sigma^{-1} \Sigma \Sigma^{-1} S_2 \Sigma^{-1} +
                \Sigma \Sigma^{-1} \Sigma \Sigma^{-1} S_2 \Sigma^{-1}           
                 }\\
        &= \sigma^2 \left(
                \expect{S_1 \Sigma^{-1} S_1 \Sigma^{-1} S_2 \Sigma^{-1}} -
                I - I + I \right).
\end{align*}
Applying Theorem~\ref{thm:wishart_one} we get $\expect{S_1 \Sigma^{-1} S_1 \Sigma^{-1} S_2 \Sigma^{-1} } = (2+p) I$.

\hfill$\Box$ 

\section{Proof of Proposition~\ref{prop:Ridge_second_moments}}
\label{apx:proof_Ridge_second_moments}

Using the notation from Appendix~\ref{apx:proof_OLS_second_moments} we set up some quantities that will be reused throughout the computation. We start with some general results assuming $Var[X] = \Sigma$, but eventually restrict the results to $Var[X] = I$ for simplicity.

When $Var[X] = \Sigma$, then
$
\nabla R(\theta) = - \Sigma (\theta_0-\theta) + \lambda \theta ;\;
\invInf = (\Sigma+\lambda I)^{-1} ;\;
W_\theta \equiv 0;\: 
\dot{f}_i = - S_i (\theta_0-\theta^*) - X_i \varepsilon_i + \lambda \theta^*; \;
\ddot{f}_i = S_i + \lambda I.
$
Adding the simplifying assumption that $\Sigma=I$ then
$
\theta^* = \lambda_{0,1} \theta_0 ;\;
\invInf = \lambda_{0,1} I ;\;
W_\theta \equiv 0;\: 
\dot{f}_i = \lambda_{1,1} (I-S_i) \theta_0 - X_i \varepsilon_i  ; \;
\ddot{f}_i = S_i + \lambda I ;\;
V_i = S_i - I_p ;\:
d_i = \lambda_{1,2} (I-S_i) \theta_0 - \lambda_{0,1} X_i \varepsilon_i.
$

Starting with the $\gamma_0$ matrix.
Since $W_\theta \equiv 0$, then like in the OLS case, $\delta =  \invInf (\expect{V_1 d_1} )$.
In our ridge setup, with the Wishart properties in Theorem~\ref{thm:wishart_two}:
\begin{align*}
        \expect{V_1 d_1} &= 
        \expect{(S_1-I)(\lambda_{1,2} (I-S_1) \theta_0 - \lambda_{0,1} X_1 \varepsilon_1)} 
        = \lambda_{1,2}  \expect{2 S_1 - S_1^2 - I} \theta_0 
        = - \lambda_{1,2} (1+p) \theta_0.
\end{align*}
Plugging $\delta$ in the definition of $\gamma_0$ in Eq.(\ref{eq:gammas}) we get $\gamma_0=\lambda_{2,6} (1+p)^2 B$.

Moving to the $\gamma_1$ matrix.
From \citep[page 374]{rilstone_second-order_1996} we have $\gamma_1 = \invInf \expect{\dot{f}_1 \dot{f}'_1} \invInf$.
In our ridge setup, with the Wishart properties in Theorem~\ref{thm:wishart_two}:
\begin{align*}
        \expect{\dot{f}_1 \dot{f}'_1} &= 
        \expect{(\lambda_{1,1} (I-S_1) \theta_0 - X_1 \varepsilon_1)(\lambda_{1,1} (I-S_1) \theta_0 - X_1 \varepsilon_1)'} \\
        &= \expect{ \lambda_{2,2} (I-S_1)B(I-S_1) + \varepsilon_1^2 S_1} 
                =\lambda_{2,2} (B+A) + \sigma^2 I,
\end{align*}
so that $\gamma_1 = \lambda_{2,4} (B+A) + \lambda_{0,2} \sigma^2 I$.

Moving to $\gamma_2$, from \citep[Lemma 3.3]{rilstone_second-order_1996} we have 
$\gamma_2= \invInf \left(-\expect{V_1 d_1 d'_1}  \right)$, and 
\begin{align*}
        \expect{V_1 d_1 d'_1} &= \expect{
                (S_1-I)
                (\lambda_{1,2} (I-S_1) \theta_0 - \lambda_{0,1} X_1 \varepsilon_1)
                (\lambda_{1,2} (I-S_1) \theta_0 - \lambda_{0,1} X_1 \varepsilon_1)'
        } \\
        &= \expect{(S_1-I) (\lambda_{2,4}(I-S_1) B (I-S_1) + \lambda_{0,2} \varepsilon_1^2 S_1 )} .
\end{align*}
Opening parenthesis and calling upon the Wishart properties in Theorem~\ref{thm:wishart_two}:
$
        \expect{S_1 B} = B  \;;
        \expect{S_1^2 B} = (2+p) B \;;
        \expect{S_1 B S_1} = 2B + A  \;;
        \expect{S_1 ^2B S_1} = (8+2p) B + (4+p)A.
$
Collecting terms: 
$$\gamma_2 = - \lambda_{0,3} \sigma^2 (1+p) I - \lambda_{2,5} ((4+p)B+(3+p)A).$$

Moving to $\gamma_3$.
Calling upon \citep[Lemma 3.3]{rilstone_second-order_1996}:
\begin{align*}
        \gamma_3= \invInf \left(
                \expect{V_1 d_1 d'_2 V'_2} +
                \expect{V_1 d_2 d'_1 V'_2} +
                \expect{V_1 d_2 d'_2 V'_1}
        \right) \invInf  .
\end{align*}
In this ridge setup, none of these terms cancel. 
The computations follow the same lines as for the previous matrices. 
The same holds for $\gamma_4$ for which $W_\theta \equiv 0$ and higher derivatives vanish, thus
\begin{align*}
        \gamma_4= \invInf \left(
                \expect{V_1 \invInf V_1 d_2 d'_2 } +
                \expect{V_1 \invInf V_2 d_1 d'_2 } +
                \expect{V_1 \invInf V_2 d_2 d'_1 } 
        \right) .
\end{align*}

\hfill$\Box$ 

\section{Properties of Wishart Matrices}
\label{apx:wishart_matrices}

In this section we collect some properties of Wishart matrices used in this work.
For classical results we provide appropriate references.
For results we did not find in the literature, we present their proofs.

\begin{theorem}
\label{thm:wishart_one}
Let $S_1$ and $S_2$ be independent $Wishart_p(1,\Sigma)$ distributed, random matrices, and let $B$ be a fixed symmetric $p \times p$ matrix. It follows that:
(i) $\expect{S_1} = \Sigma$;
(ii) $\expect{S_1 B S_1} =  2 \Sigma B \Sigma +  \Tr(\Sigma B)\Sigma $,
and 
(iii) $ \expect{S_1 B S_2 B S_1} = 2 \Sigma B \Sigma B \Sigma + \Tr(B \Sigma B \Sigma)\Sigma$.
\end{theorem}

\begin{proof}
The first two statements are simply an application of \citep[Theorem 2.2.5]{fujikoshi_multivariate_2010}.
To prove (iii) we write $S^1_{ij}=(S_1)_{ij}$ and recall the independence between $S_1$ and $S_2$ to get
\begin{align}
\label{eq:wishart_decomposition_one}
\begin{split}
        \expect{(S_1 B S_2 B S_1)_{ij}} &= \sum_{stkl} \expect{S^1_{is} B_{st} S^2_{tk} B_{kl} S^1_{lj}} 
        = \sum_{stkl} B_{st} B_{kl} \expect{S^1_{is} S^1_{lj}} \expect{S^2_{tk}}       .
\end{split}
\end{align}
Calling upon \citep[Theorem 3.3.3]{gupta_matrix_1999} we can represent $S$ as $S=X X'$ where $X \sim \mathcal{N}_p(0, \Sigma)$, so that 
$\expect{S_{is} S_{lj}} = \expect{X_i X_s X_l X_j }$  and 
$\expect{S_{tk}} = \expect{X_t X_k }$.
Now calling upon Isserlis' Theorem
\begin{align*}
        \expect{X_i X_j X_k X_l } = \expect{X_i X_k} \expect{X_j X_l} + 
                \expect{X_i X_l} \expect{X_j X_k} + 
                \expect{X_i X_j} \expect{X_k X_l} ,
\end{align*} 
which in our case equals
\begin{align}
\label{eq:gaussian_fourth_monent}
        \expect{X_i X_j X_k X_l } &= 
        \Sigma_{ik} \Sigma_{jl} + \Sigma_{il} \Sigma_{jk} + \Sigma_{ij} \Sigma_{kl}.
\end{align} 
Eq.(\ref{eq:wishart_decomposition_one}) thus yields
\begin{align*}
        \sum_{stkl} B_{st} B_{kl} \expect{S^1_{is} S^1_{lj}} \expect{S^2_{tk}}  = 
        \sum_{stkl} B_{st} B_{kl} 
        (\Sigma_{il} \Sigma_{sj} + \Sigma_{ij} \Sigma_{sl} + \Sigma_{is} \Sigma_{lj}) 
        (\Sigma_{tk}).
\end{align*}
The third statement in the theorem follows by rearranging into matrix notation.

\end{proof}

\begin{theorem}
\label{thm:wishart_two}

Let $S_1$ and $S_2$ be independent $Wishart_p(1,I_p)$ distributed, random matrices, and let $B$ be a fixed symmetric $p \times p$ matrix. It follows that:

\begin{tabular}{rlrl}
(1) &$\expect{S_1 S_2 B S_1 S_2} =  (p+6) B + 2 \Tr(B)I $.& (2) & $ \expect{S_1 S_2 B S_1} = 2 B + \Tr(B)I $. \\ 
(3) &$\expect{S_1 S_2 B S_2 S_1} = 4 B + (4+p) \Tr(B) I $.& (4) & $ \expect{S_1^2 B S_1} = (8+2p) B + (4+p) \Tr(B) I $. \\ 
(5) &$ \expect{S_1 S_2 S_1 B S_2} = (6+p) B + \Tr(B) I $.& (6) & $\expect{S_1 S_2^2 B S_1} = (4+2p) B + (2+p) \Tr(B) I $. \\ 
\end{tabular} 
\end{theorem}
\begin{proof}
We present only the proof of the first statement. The others are proved using the same arguments. 
We also note that these arguments can be used for the more general $Wishart_p(n,\Sigma)$ matrices.
We now write $S^1_{ij}=(S_1)_{ij}$ and recall the independence between $S_1$ and $S_2$ to get
\begin{align}
\label{eq:wishart_decomposition_two}
\begin{split}
        \expect{(S_1 S_2 B S_1 S_2)_{ij}} &= \sum_{stkl} \expect{S^1_{is} S^2_{st} B_{tk} S^1_{kl} S^2_{lj}} 
        = \sum_{stkl} B_{tk} \expect{S^1_{is} S^1_{kl}} \expect{S^2_{st} S^2_{lj}}.       
\end{split}
\end{align}
The mean $\expect{S_{is} S_{kl}} = \expect{X_i X_s X_k X_l }$ is given in Eq.(\ref{eq:gaussian_fourth_monent}) which in the case where $\Sigma=I$ simplifies into 
\begin{align*}
        \expect{X_i X_j X_k X_l } = \delta_{ik} \delta_{jl} + \delta_{il} \delta_{jk} + \delta_{ij} \delta_{kl}.
\end{align*} 
Eq.(\ref{eq:wishart_decomposition_two}) thus yields
\begin{align*}
        \sum_{stkl} B_{tk} 
        (\delta_{ik} \delta_{sl} + \delta_{il} \delta_{sk} + \delta_{is} \delta_{kl}) 
        (\delta_{sl} \delta_{tj} + \delta_{sj} \delta_{tl} + \delta_{st} \delta_{lj})
\end{align*}
The first statement in the theorem is recovered by rearranging into matrix notation.

\end{proof}


\section{High-Dimensional Regime}

\subsection{Proof of Lemma~\ref{lemma:perturbation_result}}
\label{apx:proof_perturbation_result}

Eq.(\ref{eq:asymptotic_representation}) is simply taken from \citet[Result 1]{bean_optimal_2013}, stated here for completeness.
From this equation it follows that the mean squared error of the centralized solution is
\begin{align}
\begin{split}
\label{eq:central_highDim_MSE}
        MSE[\centralized,\theta^*] =
        \mathbb{E}[\|\hat\theta_N-\theta^*\|_2^2] = 
        r^2(\kappa/m) \: \mathbb{E}[\|\Sigma^{-1/2}\xi\|^2_2] \: (1+o(1)).
\end{split}
\end{align}
In contrast, upon averaging the estimators of $m$ machines, we obtain
\begin{align}
\begin{split}
\label{eq:parallel_highDim_MSE}
        MSE[\parallelized,\theta^*] =  
        r^2(\kappa)/m  \: \mathbb{E}[\|\Sigma^{-1/2}\xi\|^2_2] \: (1+o(1)).
\end{split}
\end{align}

By comparing Eq.(\ref{eq:central_highDim_MSE}) and Eq.(\ref{eq:parallel_highDim_MSE}), 
the accuracy loss of parallelization compared to a centralized estimation is thus given by
\begin{equation}\label{eq:parallisation_cost}
        \frac{MSE[\parallelized,\theta^*]}{MSE[\centralized,\theta^*]} = 
        \frac{r^2(\kappa)/m}{r^2(\kappa/m)} \, (1+o(1)).
\end{equation}
Using perturbation analysis,
and denoting by 
$f_{[i]}:=\frac{\partial^i}{\partial t^i}f(t)$ 
the $i$-th derivative of the loss function, we now characterize the behaviour of Eq.(\ref{eq:parallisation_cost}) as $\kappa$ and $m$ vary. 
To this end, and in line with \cite{el_karoui_robust_2013} and \cite{donoho_high_2013}, we first introduce some more definitions:
\begin{align}
\label{eq:def_noise}
        \zhat := \epsilon + r(\kappa) \, \eta \; ; \quad
        \eta \sim \mathcal{N}(0,1), \; \mbox{independent of } \epsilon.
\end{align}
Furthermore 
$f(Y-X'\theta)$ is treated as a univariate function of the residual alone, and 
$$
   \prox_c(z) := \arg\min_x \left\{ f(x) + \frac1{2c}\|x-z\|^2 \right\}. 
$$
Readers familiar with optimization theory will recognize the equation as the standard proximal operator.

As shown by both \citet[Corollary 1]{el_karoui_robust_2013} and \citet[Theorem 4.1]{donoho_high_2013}, under the assumptions of Lemma~\ref{lemma:perturbation_result} the quantity $r(\kappa)$ together with a second quantity $c=c(f,\kappa)$ are the solution of the following set of two coupled non-linear equations
\begin{align}
        \mathbb{E}\left[ \frac{\mathrm{d}}{\mathrm{d} z} \prox_c(z) \Big|_{z=\zhat} \right] &= 1-\kappa,
        \label{eq:first_equation}\\
        \mathbb{E}\left[\left( \zhat - \prox_c(\zhat)\right)^{2}\right] & = \kappa r^2(\kappa) ,
        \label{eq:second_equation}
\end{align}
where the averaging operator in Eqs.(\ref{eq:first_equation})-(\ref{eq:second_equation}) is with respect to the random variable $\zhat$. 

Since Eqs.(\ref{eq:first_equation}) and (\ref{eq:second_equation}) are solvable analytically only in very specific cases, we study the limiting behaviour of the solution of this set of equations as $\kappa\to 0$. To this end we note that for $\kappa=0$, $r(0)=0$ and similarly $c(f,0)=0$ \cite[Section 4.3]{el_karoui_robust_2012}.
We first study the behaviour of the proximal operator \(\prox_c(z)\) for small values of \(c\). 

\begin{lemma}
Assume the loss function $\loss$ is three times differentiable, with a bounded third derivative, $|\loss_3|\leq K$. Then, as $c\to0$,
\begin{align}
\label{eq:prox_approx}
        prox_c(z) =
        z +\sqrt{c}w^*  =
        z- c \loss_1(z) + c^2\loss_1(z)\loss_2(z) + O(c^3).
\end{align}
\end{lemma}
\begin{proof}
Upon the change of variables $x=z+\sqrt{c}w$, Eq.(\ref{eq:prox_approx}) becomes 
\[
        prox_c(z) = z +\ \sqrt{c} \, \arg\min_w \left\{ \loss(z+\sqrt{c}w)+\frac12\|w\|^{2} \right\}.
\]
Denote by $w^*=w^*(z,c)$ the minimizer in the equation above.
Under the assumption that $\loss$ is differentiable, it is the solution of
\begin{equation}
        \sqrt{c} \loss_1 (z + \sqrt{c} w) + w=0.
\end{equation}
Next, we make an exact 2-term Taylor expansion of $\loss_1$ around the value $z$, with a remainder term involving the third derivative of \(\loss\). Then at $w^*$ we have
\begin{align}
\begin{split}
        \label{eq:quadratic_w_star}
        0= w^* +
         \sqrt{c} \left[\loss_1(z)+\sqrt{c} w^* \loss_2(z) + \tfrac12 c \, (w^*)^2 \loss_3(\tilde{w})\right] ,
\end{split}
\end{align}
where $\tilde{w}$ is some point in the interval $[0,w^*]$.
Eq. (\ref{eq:quadratic_w_star}) is a singular quadratic equation in $w^*$, although implicit since \(\tilde{w}\) depends on $w^*$ as well. It has two solutions, one that explodes to $\infty$ as $c\to0$ and the other, relevant to us, that tends to zero as $c\to 0$. Under the assumption that $|\loss_3|\leq K$ we have that
\begin{align}
\label{eq:w_star}
        w^*(z)=& -\frac{\sqrt{c}\loss_1(z)}{1+c\loss_2(z)} +O(c^2 \sqrt{c}) \\
      =& - \sqrt{c}\loss_1(z) + 
        c\sqrt{c}\loss_2(z)\loss_1(z) + O(c^2 \sqrt{c}). \nonumber
\end{align} 
Inserting Eq. (\ref{eq:w_star}) into the definition of the prox function concludes the proof.
\end{proof}

\subsection{Approximating The Residual Noise Equations}

Next, we study the form of Eqs.(\ref{eq:first_equation}) and (\ref{eq:second_equation}) as $\kappa\to 0$. 
As \(r(\kappa)\to0\) the distribution of the random variable $\zhat := \epsilon+r(\kappa) \eta$ converges to that of $\epsilon$. The following lemma quantifies how averaging with respect to $\zhat$ is related to averaging with respect to $\epsilon$.

\begin{lemma} Let $g:\mathbb{R}\to\mathbb{R}$ be a smooth differentiable function, with bounded fourth derivative, $|g_4(x)|\leq K$. 
Let \(\zhat:=\epsilon+r\eta\) be defined as in Eq.(\ref{eq:def_noise}). 
Then, as $r\to 0$,
\begin{equation}\label{eq:E_xi_epsilon}
        \mathbb{E}_{\zhat}[g(\zhat)] = \mathbb{E}_\epsilon[g(\epsilon)] + 
        \frac12 \, r^2 \, \mathbb{E}_\epsilon[g_2(\epsilon)] + O(r^4).
\end{equation}
\end{lemma}
\begin{proof}
By definition of the random variable \(\zhat\) and of the expectation operator,
\begin{eqnarray}
        \label{eq:E_g_xi}
\mathbb{E}_{\zhat} \left[
        g(\zhat)
 \right] &=& \iint g(\varepsilon + r \eta) dF_\varepsilon dF_\eta ,
\end{eqnarray}
where $dF_\varepsilon $ and $dF_\eta$ are the CDFs of the random variables $\varepsilon$, and $\eta$, respectively. 
Making a Taylor expansion of \(g\) up to fourth order gives
\begin{align}
\begin{split}
 g(\epsilon+r\eta)&= 
        g(\varepsilon) + 
        r \, \eta \,  g_1(\epsilon) + 
        \frac{1}{2} \, r^2 \, \eta^2 \, g_2(\epsilon)+
        \frac16 \, r^3 \, \eta^3 \, g_3(\epsilon) +
        \frac1{24} \, r^4 \, \eta^4 \, g_4(\tilde \zhat) ,
\end{split}
\end{align}
where $\tilde\zhat$ is an intermediate point in the interval $[\epsilon,\epsilon+r\eta]$. As $\eta$ is symmetrically distributed, upon inserting this expansion into Eq.(\ref{eq:E_g_xi}), odd terms cancel. Also, given that the fourth derivative is bounded, the remainder term is indeed \(O(r^4)\), and Eq. (\ref{eq:E_xi_epsilon}) follows. 
\end{proof}

We now arrive at our main result regarding the solution of the system of equations (\ref{eq:first_equation}) and (\ref{eq:second_equation}).

\begin{theorem}
As $\kappa\to0$, the solution of the system of Equations (\ref{eq:first_equation})-(\ref{eq:second_equation}) admits the following asymptotic form
\begin{align}
        \label{eq:assumed_form}
\begin{split}
c(\kappa) &= \kappa \, c_1 + \kappa^2 \, c_2 + O(\kappa^3) \;,\\
r^2(\kappa) &= \kappa \, r_1 + \kappa^2 \, r_2 +O(\kappa^3) \;.
\end{split}
\end{align}
The coefficients $c_1,c_2,r_1,r_2$ are given by
\begin{equation}
c_1 = \frac{1}{A_2}, \quad
c_2 = \frac{T_1}{A_2^3} - \frac{B_1 A_4}{A_2^4}, \qquad
r_1 = \frac{B_1}{A_2^2}, \quad  
r_2 = \frac{3 B_1 T_1}{A_2^4} - \frac{2 B_1^2 A_4}{A_2^5} + \frac{2 B_2}{A_2^3},
        \label{eq:c12_r12}
\end{equation}
where 
\begin{align}
\label{eq:coeffs}
        A_2 = \mathbb{E}_\varepsilon [f_{[2]}], \:
        A_4 = \mathbb{E}_\varepsilon [1/2 \, f_{[4]}], \:
        T_1 = \mathbb{E}_\varepsilon [f_{[2]}^2 + f_{[1]} f_{[3]}], \: 
        B_1 = \mathbb{E}_\varepsilon [f_{[1]}^2], \: 
        B_2 = \mathbb{E}_\varepsilon [f_{[1]}^2 f_{[2]}],
\end{align}
and expectations are taken only with respect to the distribution of $\epsilon$.

\end{theorem}

\begin{proof}
As common in applied mathematics, when analyzing the solutions of equations with a small parameter, we assume the solution is of the form (\ref{eq:assumed_form}), show that this expansion is consistent with the form of the equations, and extract explicit expressions for the coefficients by comparing terms of equal order in \(\kappa\) ~\citep{bender_advanced_1999}. 

To this end, let us first simplify the original set of equations, in the limit as $\kappa\to 0$. Inserting (\ref{eq:prox_approx}) into (\ref{eq:first_equation})
yields
\[
\kappa = \mathbb{E}_{\zhat}\left[c\loss_2(\zhat)-c^2(\loss_2^2(\zhat)+\loss_1(\zhat)\loss_3(\zhat))
\right] + O(c^3).
\]
Since to leading order $c=O(\kappa)$, the error in the above equation is $O(\kappa^3)$. Next, using Eq. (\ref{eq:E_xi_epsilon}), the terms on the right hand side may be further approximated as follows, 
\begin{align*}
        & \mathbb{E}_\zhat[c\loss_2(\zhat)] =
         c\mathbb{E}_\epsilon[\loss_2] + \frac12 cr^2\mathbb{E}[\loss_4] + O(cr^4) ,  \\
        & \mathbb{E}_\zhat[c^2(\loss_2^2(\zhat)+\loss_1(\zhat)\loss_3(\zhat)] =    
         c^2\mathbb{E}_\epsilon[\loss_2^2+\loss_1\loss_3] +O(c^2r^2).
\end{align*}
Given the assumed expansion for \(c(\kappa)\) and \(r^{2}(\kappa)\), 
both error terms are $O(\kappa^3)$ and
 the first equation reads 
\begin{align}
\begin{split}
\label{eq:1st_eq_kappa}
        \kappa = c\mathbb{E}_\epsilon[\loss_2] + 
        \tfrac12 c{r^2}\mathbb{E}_\epsilon[\loss_4]
        -c^2\mathbb{E}_\epsilon[\loss_2^2+\loss_1\loss_3] + O(\kappa^3). 
\end{split}
\end{align}

Next, we simplify the second equation (\ref{eq:second_equation}). To this end, note that
as $\kappa\to0$ and also $c\to 0$, 
\begin{align}
\begin{split}
        (\zhat-prox_c(\zhat))^2 =
        (c\loss_1-c^2\loss_1\loss_2+O(c^3))^2 = 
        c^2\loss_1^2-2c^3\loss_1^2\loss_2 + O(\kappa^4).
\end{split}
\end{align}
Thus, Eq.(\ref{eq:second_equation}) reads
\begin{align}
\begin{split}
        \kappa r^2 = \mathbb{E}_\zhat[(\zhat-\prox_c(\zhat)^2] =
        \mathbb{E}_\zhat[c^2 \loss_1^2(\zhat)-2c^3\loss_1^2(\zhat)\loss_2(\zhat)] + O(\kappa^4) .
\end{split}
\end{align}
Using Eq.(\ref{eq:E_xi_epsilon}), we further simplify the right hand side, to read
\begin{align}
\begin{split}
     \label{eq:2nd_eq_kappa}
        \kappa r^2 = 
        c^2\mathbb{E}_\epsilon[f_{[1]}^2] + c^2r^2\mathbb{E}_\epsilon[\loss_2^2+\loss_1\loss_3] -  
         2 c^3 \mathbb{E}_\epsilon[\loss_1^2 \loss_2] + O(\kappa^4).
\end{split}
\end{align}
For ease of notation, we thus write the system of Equations (\ref{eq:1st_eq_kappa}) and (\ref{eq:2nd_eq_kappa}) as follows, 
\begin{align}\label{eq:clean_approximations}
\begin{split}
\kappa &= c A_2 + c r^2 A_4 -c^2 T_1 + O(\kappa^3), \\
\kappa r^2 &= c^2 B_1 + c^2 r^2 T_1 - 
        2 c^3 B_2+ 
        O(\kappa^4),
\end{split}
\end{align}
where the coefficients $A_2,A_4,B_1,B_2$ and $T_1$ are given in Eq.~(\ref{eq:coeffs}). 

To determine the coefficients \(c_1,c_2,r_1,r_2\) we now insert the expansion (\ref{eq:assumed_form}) into Eq.(\ref{eq:clean_approximations}), and compare powers of $\kappa$. This gives
\begin{align}
\kappa =& A_2 c_1\kappa  +
        A_2 c_2\kappa^2  +
        A_4 c_1 r_1 \kappa^2 - 
        T_1 c_1^2 \kappa^2 + O(\kappa^3) \label{eq:kappa1}, \\
        r_1 \kappa^2 + r_2\kappa^3 =& B_1(c_1^2\kappa^2 +2c_1 c_2\kappa^3 ) + 
        T1 c_1^2 r_1\kappa^3  -
        B_2 c_1^32\kappa^3 + O(\kappa^4). \label{eq:kappa2} 
\end{align}
Thus, comparing the \(O(\kappa)\) terms in Eq. (\ref{eq:kappa1}) gives
$
c_1 =\frac{1}{A_2}.
$
Similarly, comparing the \(O(\kappa^2)\) terms in Eq. (\ref{eq:kappa2}) yields
$
r_1 = B_1 c_1^2 = \frac{B_1}{A_2^2}.
$
Next, comparing the \(O(\kappa^2)\) terms in Eq. (\ref{eq:kappa1}) gives
\[
c_2 = \frac{T_1c_1^2 - A_4 c_1 r_1}{A_2} = \frac{T_1}{A_2^3} - \frac{B_1A_4}{A_2^4}. 
\]
Finally, comparing the $O(\kappa^3)$ terms in Eq. (\ref{eq:kappa2}) yields the expression for $r_2$ in Eq. (\ref{eq:c12_r12}).
\end{proof}

%
%

\section{Scope of Theorem~\ref{thm:fixed_p_loss}}
\label{apx:proof_scope}

The assumptions of Theorem~\ref{thm:fixed_p_loss} are standard in the M-Estimation literature and apply to many common learning tasks.
Here is a non-comprehensive list of examples, relevant for parallelization.


\begin{example}[Non-Linear Least Squares]
\label{eg:non_linear_ls}
Here the loss function is $f(Z,\theta):=(Y-g_\theta(X))^2$, where \(g_{\theta}\)
is a smooth function indexed by a parameter \(\theta\). 
Applicability of Theorem~\ref{thm:fixed_p_loss} to this example, under suitable regularity conditions on the family $g_\theta$ is discussed in \cite[Example 5.27]{vaart_asymptotic_1998}.
\end{example}

\begin{example}[Non-Linear Quantile Regression]
\label{eg:quantile_regression}
The loss function corresponding to a quantile level $\tau$, is the tilted hinge loss $f(Z,\theta):=|(Y-g_\theta(X))(\tau-I_{\{(Y-g_\theta(X))<0\}})|$.
This case is similar to Example~\ref{eg:non_linear_ls} with the complication that the loss function has a single non differentiability point. 
Assumption~\ref{as:as_differentiability} is still satisfied for continuous $X$, since the probability of $f(Z,\theta)$ to be non-differentiable at $\theta^*$ is $0$.
\end{example}

\begin{example}[Robust Regression]
Consider non-linear robust regression with the Huber loss function, 
\begin{align*}
 f(Z,\theta) :=
  \begin{cases}
   \frac{1}{2} (Y-g_\theta(X))^2 & \text{if } |Y-g_\theta(X)| \leq \delta \\
   \delta \left( |Y-g_\theta(X)|-\frac{1}{2}\delta \right)       & \text{otherwise } 
  \end{cases}
\end{align*}
The linear case with $g_\theta(x)=X'\theta$ is discussed in \citep[Examples 5.28 \& 5.29]{vaart_asymptotic_1998}.
Theorem \ref{thm:fixed_p_loss} also applies to the  non-linear case, under appropriate assumptions on $g_\theta$. 
\end{example}

\begin{example}[Binary Regression]
\label{eg:ml_binary_regression}
Here the binary response $Y$ is assumed to arise from the generative model $P(Y|X)=\Psi(X'\theta)$, with a known and sufficiently regular $\Psi$. With the loss function taken as the negative log likelihood,
this case is discussed in \citep[Example 5.40]{vaart_asymptotic_1998}.
This setting covers popular generalized linear models such as Logistic, Probit, and Complementary Log-Log regression.
\end{example}

\begin{example}[SVM]
For $Y \in \{-1,1\}$ and $f(Z,\theta):=\max \{0,1- Y \, X'\theta \} + \frac{\lambda}{2} \Vert \theta \Vert^2_2$ we get the SVM problem with $l_2$ regularization. 
This problem satisfies the risk regularity assumption \ref{as:taylor} \citep[e.g.][]{shamir_is_2012}. 
The differentiability of the empirical risk \ref{as:as_differentiability} is settled, for continuous $X$, as in Example~\ref{eg:quantile_regression}.
\end{example}


\bibliographystyle{abbrvnat}
\bibliography{Part1}

\begin{thebibliography}{33}
\providecommand{\natexlab}[1]{#1}
\providecommand{\url}[1]{\texttt{#1}}
\expandafter\ifx\csname urlstyle\endcsname\relax
  \providecommand{\doi}[1]{doi: #1}\else
  \providecommand{\doi}{doi: \begingroup \urlstyle{rm}\Url}\fi

\bibitem[Achutegui et~al.(2014)Achutegui, Crisan, Miguez, and
  Rios]{achutegui_simple_2014}
K.~Achutegui, D.~Crisan, J.~Miguez, and G.~Rios.
\newblock A simple scheme for the parallelization of particle filters and its
  application to the tracking of complex stochastic systems.
\newblock \emph{ArXiv e-prints}, 2014.

\bibitem[Bean et~al.(2013)Bean, Bickel, Karoui, and Yu]{bean_optimal_2013}
D.~Bean, P.~J. Bickel, N.~E. Karoui, and B.~Yu.
\newblock Optimal {M}-estimation in high-dimensional regression.
\newblock \emph{Proceedings of the National Academy of Sciences}, pages
  14563--14568, 2013.

\bibitem[Bekkerman et~al.(2011)Bekkerman, Bilenko, and
  Langford]{bekkerman_scaling_2011}
R.~Bekkerman, M.~Bilenko, and J.~Langford.
\newblock \emph{Scaling up {Machine} {Learning}: {Parallel} and {Distributed}
  {Approaches}}.
\newblock Cambridge University Press, Cambridge; New York, 2011.

\bibitem[Bender and Orszag(1999)]{bender_advanced_1999}
C.~M. Bender and S.~A. Orszag.
\newblock \emph{Advanced {Mathematical} {Methods} for {Scientists} and
  {Engineers}: {Asymptotic} {Methods} and {Perturbation} {Theory}}.
\newblock Springer, New York, 1999.

\bibitem[Dean and Ghemawat(2008)]{dean_mapreduce:_2008}
J.~Dean and S.~Ghemawat.
\newblock {MapReduce}: {Simplified} {Data} {Processing} on {Large} {Clusters}.
\newblock \emph{Commun. ACM}, 51\penalty0 (1):\penalty0 107--113, 2008.

\bibitem[Devroye et~al.(1997)Devroye, Gyorfi, and
  Lugosi]{devroye_probabilistic_1997}
L.~Devroye, L.~Gyorfi, and G.~Lugosi.
\newblock \emph{A {Probabilistic} {Theory} of {Pattern} {Recognition}}.
\newblock Springer, New York, 1997.

\bibitem[Donoho and Montanari(2013)]{donoho_high_2013}
D.~Donoho and A.~Montanari.
\newblock High {Dimensional} {Robust} {M}-{Estimation}: {Asymptotic} {Variance}
  via {Approximate} {Message} {Passing}.
\newblock \emph{ArXiv e-prints}, 2013.

\bibitem[Doss and Sethuraman(1989)]{doss_price_1989}
H.~Doss and J.~Sethuraman.
\newblock The {Price} of {Bias} {Reduction} when there is no {Unbiased}
  {Estimate}.
\newblock \emph{The Annals of Statistics}, 17\penalty0 (1):\penalty0 440--442,
  1989.

\bibitem[El~Karoui(2013)]{el_karoui_asymptotic_2013}
N.~El~Karoui.
\newblock Asymptotic behavior of unregularized and ridge-regularized
  high-dimensional robust regression estimators: rigorous results.
\newblock \emph{ArXiv e-prints}, 1311:\penalty0 2445, 2013.

\bibitem[El~Karoui et~al.(2012)El~Karoui, Bickel, Bean, Lim, and
  Yu]{el_karoui_robust_2012}
N.~El~Karoui, P.~Bickel, D.~Bean, C.~Lim, and B.~Yu.
\newblock On robust regression with high- dimensional predictors.
\newblock Technical report, University of California, Berkeley, 2012.

\bibitem[El~Karoui et~al.(2013)El~Karoui, Bean, Bickel, Lim, and
  Yu]{el_karoui_robust_2013}
N.~El~Karoui, D.~Bean, P.~J. Bickel, C.~Lim, and B.~Yu.
\newblock On robust regression with high-dimensional predictors.
\newblock \emph{Proceedings of the National Academy of Sciences}, page
  14557�14562, 2013.

\bibitem[Feng et~al.(2014)Feng, Xu, and Mannor]{feng_distributed_2014}
J.~Feng, H.~Xu, and S.~Mannor.
\newblock Distributed {Robust} {Learning}.
\newblock \emph{ArXiv e-prints}, 2014.

\bibitem[Fujikoshi et~al.(2010)Fujikoshi, Ulyanov, and
  Shimizu]{fujikoshi_multivariate_2010}
Y.~Fujikoshi, V.~V. Ulyanov, and R.~Shimizu.
\newblock \emph{Multivariate {Statistics}: {High}-{Dimensional} and
  {Large}-{Sample} {Approximations}}.
\newblock Wiley, 2010.

\bibitem[Gupta and Nagar(1999)]{gupta_matrix_1999}
A.~K. Gupta and D.~K. Nagar.
\newblock \emph{Matrix {Variate} {Distributions}}.
\newblock Chapman and Hall, Boca Raton, FL, 1999.

\bibitem[Hsu and Sabato(2013)]{hsu_loss_2013}
D.~Hsu and S.~Sabato.
\newblock Loss minimization and parameter estimation with heavy tails.
\newblock \emph{ArXiv e-prints}, 2013.

\bibitem[Huber(1973)]{huber_robust_1973}
P.~J. Huber.
\newblock Robust {Regression}: {Asymptotics}, {Conjectures} and {Monte}
  {Carlo}.
\newblock \emph{The Annals of Statistics}, 1\penalty0 (5):\penalty0 799--821,
  1973.

\bibitem[Kim(2006)]{kim_higher_2006}
K.~I. Kim.
\newblock Higher {Order} {Bias} {Correcting} {Moment} {Equation} for
  {M}-{Estimation} and {Its} {Higher} {Order} {Efficiency}.
\newblock \emph{Research Collection School Of Economics}, 2006.

\bibitem[Liu and Ihler(2014)]{liu_distributed_2014_2}
Q.~Liu and A.~T. Ihler.
\newblock Distributed estimation, information loss and exponential families.
\newblock In \emph{Advances in Neural Information Processing Systems 27}, pages
  1098--1106. Curran Associates, Inc., 2014.

\bibitem[Mcdonald et~al.(2009)Mcdonald, Mohri, Silberman, Walker, and
  Mann]{mcdonald_efficient_2009}
R.~Mcdonald, M.~Mohri, N.~Silberman, D.~Walker, and G.~S. Mann.
\newblock Efficient {Large}-{Scale} {Distributed} {Training} of {Conditional}
  {Maximum} {Entropy} {Models}.
\newblock In Y.~Bengio, D.~Schuurmans, J.~D. Lafferty, C.~K.~I. Williams, and
  A.~Culotta, editors, \emph{Advances in {Neural} {Information} {Processing}
  {Systems} 22}, pages 1231--1239. Curran Associates, Inc., 2009.

\bibitem[Meng et~al.(2012)Meng, Wiesel, and Hero]{meng_distributed_2012}
Z.~Meng, A.~Wiesel, and A.~Hero.
\newblock Distributed principal component analysis on networks via directed
  graphical models.
\newblock In \emph{IEEE International Conference on Acoustics, Speech and
  Signal Processing (ICASSP)}, pages 2877--2880, 2012.

\bibitem[Rieder(2012)]{rieder_robust_2012}
H.~Rieder.
\newblock \emph{Robust {Asymptotic} {Statistics}: {Volume} {I}}.
\newblock Springer, New York, NY, 2012.

\bibitem[Rilstone et~al.(1996)Rilstone, Srivastava, and
  Ullah]{rilstone_second-order_1996}
P.~Rilstone, V.~K. Srivastava, and A.~Ullah.
\newblock The second-order bias and mean squared error of nonlinear estimators.
\newblock \emph{Journal of Econometrics}, 75\penalty0 (2):\penalty0 369--395,
  1996.

\bibitem[Rosset and Zhu(2007)]{rosset_piecewise_2007}
S.~Rosset and J.~Zhu.
\newblock Piecewise linear regularized solution paths.
\newblock \emph{The Annals of Statistics}, 35\penalty0 (3):\penalty0
  1012--1030, 2007.

\bibitem[Shalev-Shwartz and Ben-David(2014)]{shalev-shwartz_understanding_2014}
S.~Shalev-Shwartz and S.~Ben-David.
\newblock \emph{Understanding {Machine} {Learning}: {From} {Theory} to
  {Algorithms}}.
\newblock Cambridge University Press, New York, NY, 2014.

\bibitem[Shalev-Shwartz and Srebro(2008)]{shalev-shwartz_svm_2008}
S.~Shalev-Shwartz and N.~Srebro.
\newblock {SVM} {Optimization}: {Inverse} {Dependence} on {Training} {Set}
  {Size}.
\newblock \emph{Proceedings of the 25th International Conference on Machine
  Learning}, pages 928--935, 2008.

\bibitem[Shamir(2012)]{shamir_is_2012}
O.~Shamir.
\newblock Is averaging needed for strongly convex stochastic gradient descent.
\newblock \emph{Open problem presented at COLT}, 2012.

\bibitem[Shamir et~al.(2013)Shamir, Srebro, and
  Zhang]{shamir_communication_2013}
O.~Shamir, N.~Srebro, and T.~Zhang.
\newblock Communication {Efficient} {Distributed} {Optimization} using an
  {Approximate} {Newton}-type {Method}.
\newblock \emph{ArXiv e-prints}, 2013.

\bibitem[Shvachko et~al.(2010)Shvachko, Kuang, Radia, and
  Chansler]{shvachko_hadoop_2010}
K.~Shvachko, H.~Kuang, S.~Radia, and R.~Chansler.
\newblock The {Hadoop} {Distributed} {File} {System}.
\newblock \emph{Proceedings of IEEE 26th Symposium on Mass Storage Systems and
  Technologies (MSST)}, pages 1--10, 2010.

\bibitem[Vaart(1998)]{vaart_asymptotic_1998}
A.~W. v.~d. Vaart.
\newblock \emph{Asymptotic {Statistics}}.
\newblock Cambridge University Press, 1998.

\bibitem[Zaharia et~al.(2010)Zaharia, Chowdhury, Franklin, Shenker, and
  Stoica]{zaharia_spark:_2010}
M.~Zaharia, M.~Chowdhury, M.~J. Franklin, S.~Shenker, and I.~Stoica.
\newblock Spark: cluster computing with working sets.
\newblock In \emph{Proceedings of the 2nd {USENIX} conference on {Hot} topics
  in cloud computing}, pages 10--10, 2010.

\bibitem[Zhang et~al.(2013{\natexlab{a}})Zhang, Duchi, and
  Wainwright]{zhang_divide_2013}
Y.~Zhang, J.~Duchi, and M.~Wainwright.
\newblock Divide and conquer kernel ridge regression.
\newblock In \emph{Conference on Learning Theory}, pages 592--617,
  2013{\natexlab{a}}.

\bibitem[Zhang et~al.(2013{\natexlab{b}})Zhang, Duchi, and
  Wainwright]{zhang_communication-efficient_2013}
Y.~Zhang, J.~C. Duchi, and M.~J. Wainwright.
\newblock Communication-efficient {Algorithms} for {Statistical}
  {Optimization}.
\newblock \emph{J. Mach. Learn. Res.}, 14\penalty0 (1):\penalty0 3321--3363,
  2013{\natexlab{b}}.

\bibitem[Zinkevich et~al.(2010)Zinkevich, Weimer, Li, and
  Smola]{zinkevich_parallelized_2010}
M.~Zinkevich, M.~Weimer, L.~Li, and A.~J. Smola.
\newblock Parallelized stochastic gradient descent.
\newblock \emph{Advances in Neural Information Processing Systems}, pages
  2595--2603, 2010.

\end{thebibliography}

\end{document}